\documentclass[journal]{IEEEtran}

\usepackage{amssymb}
\usepackage{amsmath,amsfonts}
\usepackage{amsthm}

\usepackage[caption=false,font=normalsize,labelfont=sf,textfont=sf]{subfig}
\usepackage{textcomp}
\usepackage{stfloats}
\usepackage{url}
\usepackage{verbatim}
\usepackage{graphicx}
\usepackage{cite}
\usepackage{thmtools}
\usepackage{bm}
\usepackage{xcolor}
\usepackage{hhline}
\usepackage{tabularx}

\usepackage{cleveref}

\declaretheoremstyle[
  bodyfont=\normalfont\itshape,
  headformat=\NAME\NUMBER  
]{nospacetheorem}
\declaretheorem[style=nospacetheorem,name=P]{property}

\newcommand{\reffig}[1]{Fig.~\ref{fig:#1}}

\newcommand{\refsubsec}[1]{Subsection~\ref{subsec:#1}}
\newcommand{\refeq}[1]{(\ref{eq:#1})}
\newcommand{\refprop}[1]{P\ref{prop:#1}}

\def\trans{^{\mathsf{T}}}
\def\R{\mathbb{R}}
\def\F{\mathcal{F}}

\newtheorem{theorem}{Theorem}%  meant for continuous numbers

\begin{document}

\title{Coordinated Motion Control of a Wire Arc\\ Additive Manufacturing Robotic System\\ for Multi-directional Building Parts}
\author{Fernando~Coutinho,
        Nicolas~Lizarralde~
        and~Fernando~Lizarralde,~\IEEEmembership{Senior Member,~IEEE}% <-this % stops a space
\thanks{This work was supported in part by Shell Brasil Petróleo Ltda, Empresa Brasileira de Pesquisa e Inovação Industrial (Embrapii) and Coordenação de Aperfeiçoamento de Pessoal de Nível Superior – Brasil (CAPES) – Finance Code 001.}
\thanks{F. Coutinho is with the Faculty of Electrical Engineering, Federal University of South and Southeast of Pará, Brazil, e-mail: fernando.coutinho@unifesspa.edu.br.}% <-this % stops a space
\thanks{N. Lizarralde is with the Department of Electrical, Computer, and
Systems Engineering, Rensselaer Polytechnic Institute, Troy, NY 12180
USA.}% <-this % stops a space
\thanks{F. Lizarralde is with Department of Electrical Engineering, COPPE, Federal University of Rio de Janeiro, Brazil.}
}

%}
% \affiliation[ufrj]{organization={Department of Electrical Engineering, COPPE, Federal University of Rio de Janeiro},
%             addressline={Rua Moniz Aragão Nº 360, Bloco 1, Ilha do Fundão - Cidade Universitária}, 
%             postcode={21941-97}, 
%             city={Rio de Janeiro},
%             state={RJ},
%             country={Brazil}}

% \affiliation[unifesspa]{organization={Faculty of Electrical Engineering, Federal University of South and Southeast of Pará},
%             addressline={Folha 17, Quadra 04, Lote Especial, s/n.º - Nova Marabá}, 
%             postcode={68505-080}, 
%             city={Marabá},
%             state={PA},
%             country={Brazil}}

\maketitle

\begin{abstract}
This work investigates the manufacturing of complex shapes parts with wire arc additive manufacturing (WAAM). In order to guarantee the integrity and quality of each deposited layer that composes the final piece, the deposition process is usually carried out in a flat position. However, for complex geometry parts with non-flat surfaces, this strategy causes unsupported overhangs and staircase effect, which contribute to a poor surface finishing.
Generally, the build direction is not constant for every deposited section or layer in complex geometry parts. As a result, there is an additional concern to ensure the build direction is aligned with gravity, thus improving the quality of the final part.
This paper proposes an algorithm to control the torch motion with respect to a deposition substrate as well as the torch orientation with respect to an inertial frame. The control scheme is based on task augmentation applied to an extended kinematic chain composed by two robots, which constitutes a coordinated control problem, and allows the deposition trajectory to be planned with respect to  the deposition substrate coordinate frame while aligning each layer buildup direction with gravity (or any other direction defined for an inertial frame). Parts with complex geometry aspects have been produced in a WAAM cell composed by two robots (a manipulator with a welding torch and a positioning table holding the workpiece) in order to validate the proposed approach.
\end{abstract}

\def\abstractname{Note to Practitioners}
\begin{abstract}
This work addresses the technical constraints of Wire Arc Additive Manufacturing (WAAM) regarding non-planar geometries and variable build directions. To mitigate issues like unsupported overhangs and the staircase effect, we present a closed-loop coordinated control algorithm for dual-robot systems, where a manipulator guides the torch and a secondary robot repositions the substrate. The practical advantage lies in the task-augmentation scheme that automates gravity-aligned deposition, ensuring the melt pool remains stable and flat throughout the process. This eliminates the need for complex support structures and minimizes post-processing. By synchronizing both robots, the algorithm allows operators to plan trajectories directly on the workpiece frame while the controller autonomously maintains optimal torch orientation relative to an inertial frame, significantly enhancing surface quality and operational efficiency in complex part production.
\end{abstract}



\section{Introduction}
\label{sec:introduction}
Metal additive manufacturing (MAM) is the process of producing near-net-shaped metallic parts layer-by-layer from a computer-aided design (CAD) model while maintaining predetermined internal characteristics \cite{Gibson2010}. 

Although MAM cannot be considered a novel method for manufacturing parts, since the first experiments with this technology date from the early 1980s, when it was first termed rapid prototyping \cite{Kumar2021}, it has recently attracted interest from industry and academia due to its high capability to manufacture complex structural parts that are difficult or even not possible to manufacture using conventional manufacturing techniques \cite{Kumar2021,Huang2012}. 

The parts produced with MAM are printed in layers using consumable materials, mostly metal powder or metal wire feedstock, hence, MAM can save time and material in the fabrication of parts with complex shapes if compared to conventional subtractive manufacturing technology, which starts with an oversized raw block and removes unwanted material \cite{Zhao2018NonplanarSA}. Furthermore, MAM processes usually have the flexibility to produce a numerous variety of parts using the same set of equipment and feedstock, which attracts the interest of all kinds of industries, specially those located in isolated environments with difficult access to transport spare parts (e.g., mining, oil, gas).

Among the available MAM technologies described in the ASTM F2792 standard, wire arc additive manufacturing (WAAM) stands out due to its potential capability to produce large parts with distinct geometry without the necessity of specific tools. WAAM is the combination of wire feedstock with an electric arc as a heat source, all guided by a robotic system to enable the buildup of a designed part. It uses standard robotic welding hardware, such as welding power source, torch, wire feeding system, positioning table, robot arm, and others common welding equipment \cite{Williams2016}, which makes it less expensive and more adaptable for existing production lines.

Complex geometry parts for which build direction is not constant remain a challenge for WAAM due to the unsupported overhang sections, in which the molten material may drip and spread to undesired areas. An example of a multi-directional building part (curved pipe) produced with WAAM is shown in \reffig{curved-pipe}.
\begin{figure}[ht!]
    \centering
    \includegraphics[width=0.5\columnwidth,trim={0 5cm 0 15cm},clip]{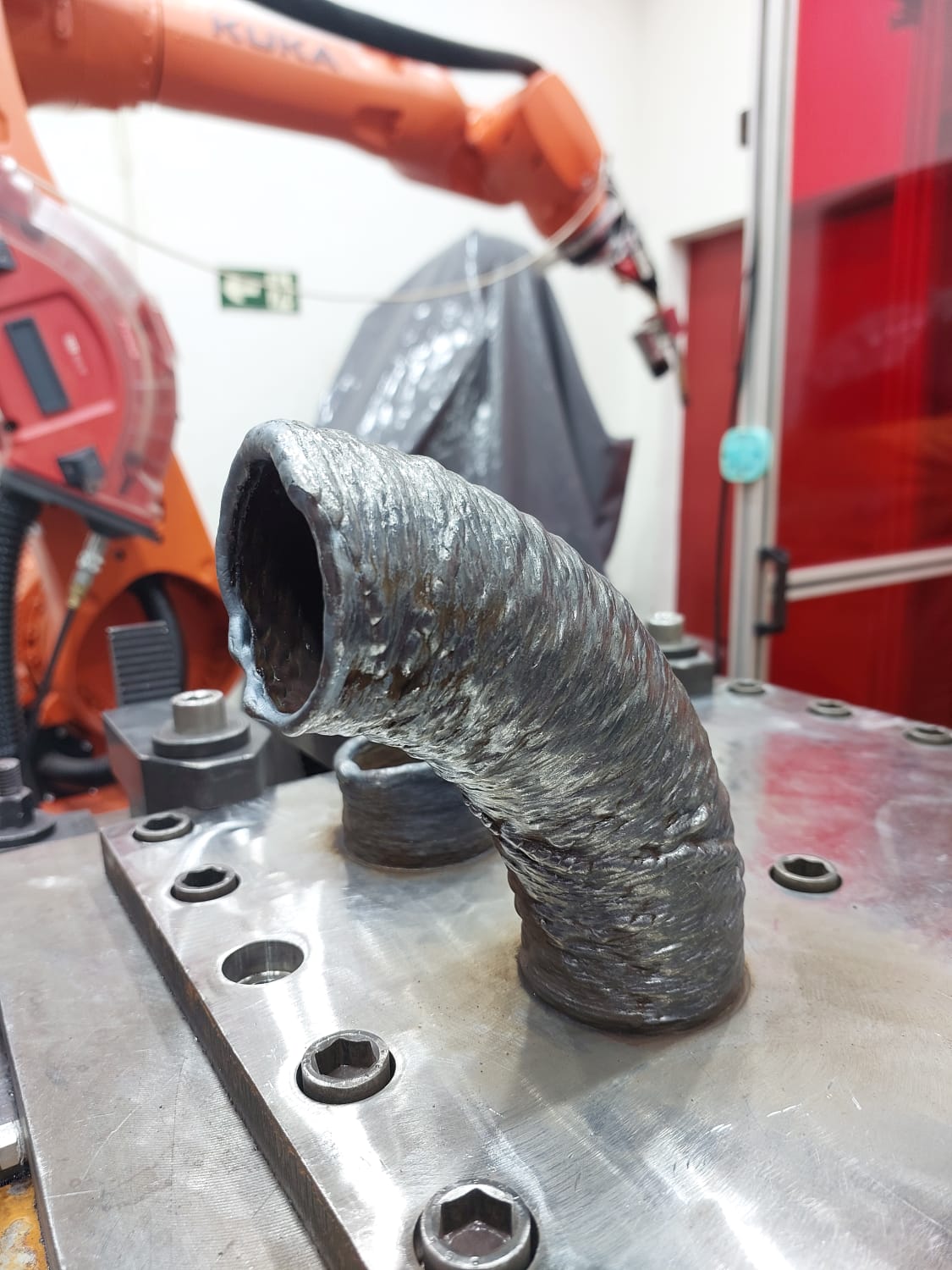}
    \caption{Multi-directional building example.}
    \label{fig:curved-pipe}
\end{figure}
Many researches investigate the possibility of producing complex parts using WAAM by a variety of techniques such as depositing using support structures to deal with overhang sections, changing the torch orientation while keeping the deposition substrate parallel to the ground \cite{Yuan2020}.

The use of support structures leads to an increase of waste and cost of post-processing \cite{Ding2016} and leads to a non-optimal manufacturing time. Changing only the torch direction with respect to the substrate produces irregular melting pool geometry, which leads to changes in the morphology of the deposited bead \cite{Hao2019} leading to a stair-stepping effect on the part surface and degrades the part’s volumetric accuracy \cite{Bhat2022}.

Conventional additive manufacturing systems are generally limited to three degrees of freedom (DOF) with fixed nozzle orientation \cite{Jiang2020}. For WAAM, the most common system utilized to carry the welding power source torch is the standard 6-DOF manipulator, mostly because the majority of power sources and industrial manipulators manufacturers already provide fully integrated solutions between these hardware as they are commonly used for industrial welding tasks. For this setup, the building of parts with unsupported overhangs, depending on the inclined or curved section of the structure, may not be achievable by WAAM as gravity affects the morphology of the molten pool \cite{Gibson2022}.

In recent years, many researchers have turned their efforts to additive manufacturing technologies with redundant robotic systems. The extra DOFs available on these systems are useful to improve the part strength and the surface quality or, as it is the case for WAAM, produce more complex parts \cite{Jiang2020,Lehman2022}.

In order to produce complex geometry parts and reduce production time, \cite{Dai2022} proposed similar offline planning algorithms for manufacturing specific rounded metal additive manufacturing parts parts with a system that coupled a 6-axis robot manipulator with a 2-axis positioning table. These methods consist of planning the deposition trajectory while using an external-axis system in order to maintain the molten pool orientation with respect to gravity. 

For curved parts offline solutions \cite{Ding2017,Dai2022,Gibson2022}, the part was either fabricated in the center of the deposition surface, aligned with the axis of rotation of the positioning table, or the substrate (e.g. metal cylinder) was centered to facilitate the part manufacturing and trajectory planning. 

A general limitation of offline planning methods is that although they allow parameters correction before each layer deposition, they cannot make trajectory or deposition parameters corrections during the deposition. For instance, the deposition bead/layer geometry could be monitored and estimated by sensors and an online closed loop control could adjust the deposition parameters (arc current, wire feed speed), deposition path and tool center point (TCP) travel speed \cite{Marcotte2025,Fang2024}.

For a WAAM system composed by two robots, which together constitute a single redundant kinematic chain, it is possible to control the motions of both the workpiece and the torch, in order to decouple the orientation between the torch and the workpiece and the orientation between the torch and an inertial frame (e.g. gravity). This strategy allows full closed loop control, which could be used for parameters corrections during deposition and facilitates the deposition trajectory planning, since the trajectory can be designed directly on the deposition substrate while keeping the building direction aligned with both the deposition orientation and an inertial frame.

In the present study, the deposition trajectory with respect to the deposition frame (positioning table workpiece) and the TCP orientation with respect to an inertial frame are treated as one single trajectory reference in order to maintain the building direction aligned with gravity during each layer deposition. Which means that, for the additive manufacturing of complex shaped parts structures with unsupported overhangs is treated as a multi robot coordinated motion problem by defining an augmented task vector which considers both the deposition trajectory and the deposition alignment with an inertial frame, solved by using the augmented Jacobian approach.

The augmented Jacobian method presents a more general approach alternative to the task-priority method applied in WAAM systems presented in \cite{Lizarralde2022}. An advantage of the augmented Jacobian method over the task priority method is its ability to utilize well-established techniques, such as damped least squares inverse with numerical filtering \cite{chiaverini}, to handle depositions even when algorithmic singularities occur, that is, when the algorithm faces conflicting tasks. In the proposed solution, a common algorithmic singularity occurs when the deposition frame and the inertial frame align their $z$ axes, causing the augmented Jacobian to lose rank.

The similarities between the augmented task algorithm and the constrained Jacobian method \cite{Pham2015} are also stated for better understanding of algorithm singularities and the asymptotic stability via Lyapunov theory is presented considering internal controllers unmodeled dynamics. 

Experimental results obtained with an integrated robotic system composed by a robot manipulator and a positioning table producing CMT-WAAM multi-directional building parts are presented to show the feasibility of the proposed method.

%%%%%%%%%%%%%%%%
%%%%%%%%%%%%%%%%
%%%%%%%%%%%%%%%%
\section{Problem overview and system description}
\label{sec:system-description}
For manufacturing parts with inclined or curved sections using WAAM, each layer build direction may be different from the previous ones. Subsequently, each layer build direction is misaligned from the gravity vector, causing unsupported overhang sections which may lead to molten metal dripping or spreading to undesired areas, or leading to stair-stepping effect on the part surface.

During the manufacturing of complex geometry parts, ideally, the layer build direction, the deposition direction and gravity are aligned, producing parts with better surface quality, leading to less deposition defects and  reducing the need for surface finishing. 

Thus, in order to increase the range of structures to be produced with WAAM it is desirable to realign the molten material normal with gravity during each layer deposition, which can be achieved by using another robot (e.g. positioning table) to reposition the workpiece while the manipulator carries the welding torch to perform the deposition. \reffig{build-direction} shows a representation of these vectors for the manufacturing of a curved part and the repositioning of the workpiece for a $n$-th layer. 
\begin{figure}[ht!]
    \centering
    \includegraphics[width=0.98\columnwidth]{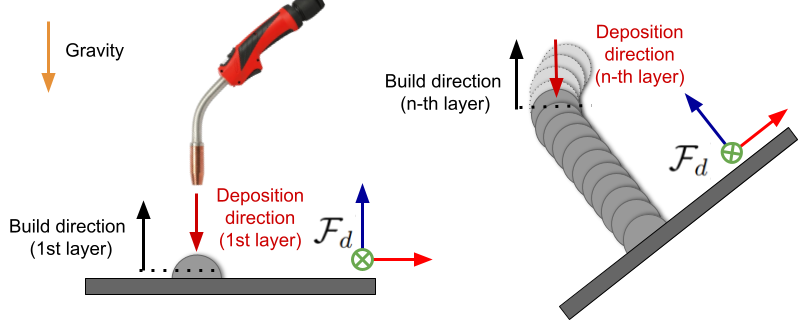}
    \caption{Build direction schematic.}
    \label{fig:build-direction}
\end{figure}

The realignment of the part building direction during the manufacturing process consists of a coordinated control motion problem between two robots. 
In this work, the robotic arm and positioning table are considered as a single $m$-DOF redundant kinematic chain. In this case, the kinematic chain begins in the workpiece surface, represented by the Deposition Frame ($\mathcal{F}_d$), goes through the table and the manipulator kinematic chain, considering the displacement between the Table Base Frame ($\mathcal{F}_{tb}$) and the Arm Base Frame ($\mathcal{F}_{ab}$), and ends up in the Torch tip Frame ($\mathcal{F}_{t}$). A schematic of the robotic WAAM cell with its frames is shown in \reffig{cell-drawing}.
\begin{figure}[ht!]
    \centering
    \includegraphics[width=0.7\columnwidth]{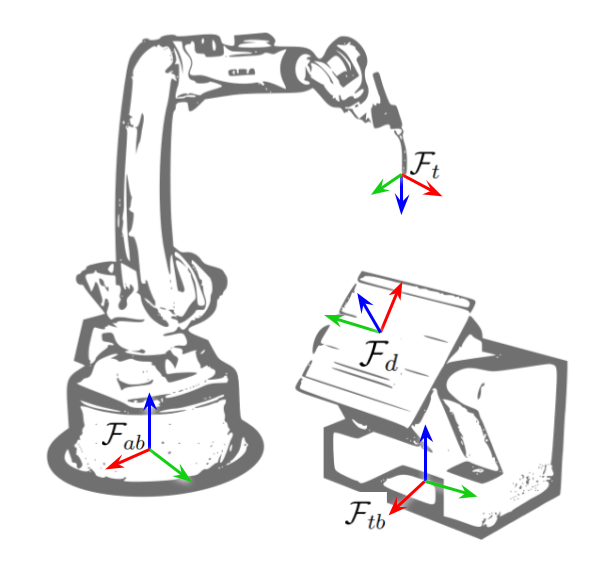}
    \caption{WAAM cell schematic}
    \label{fig:cell-drawing}
\end{figure}

Therefore, the objective is to control the complete robotic system in order to perform a deposition task, where the arm TCP, defined at the tip of the welding torch, follows a desired trajectory defined at $\mathcal{F}_d$ while the build direction which, for that matter, is the same as the deposition direction, is kept aligned with gravity to improve the surface finishing and the overall part quality. Considering that the deposition itself is a task defined for a 6-DOF operational space, any additional task justifies the utilization of a redundant robotic system.

The first task is defined as trajectory tracking problem, where the TCP pose $\bm{x}$ has to follow a desired trajectory $\bm{x}_{(d)}$ defined at $\mathcal{F}_d$. The TCP pose is defined by the forward kinematics of the full redundant system is defined as \cite{siciliano2011robotics}:
\begin{equation}
    \bm{x} = \begin{bmatrix}
    \bm{p} \\ q
    \end{bmatrix} =
    k(\bm{\theta})
    \label{eq:fk}
\end{equation}
where $\bm{x}$ is composed by the TCP position $\bm{p} \in \R^3$ and orientation $R \in SO(3)$, represented in its unit quaternion form, 
$q  = [\eta \; \; \bm{\epsilon}\trans]\trans \in \R^4$ with respect to $\mathcal{F}_d$, where $\eta \in \R$ and $\bm{\epsilon} \in \R^3$ are the scalar and vector parts of the quaternion \cite{siciliano2011robotics}. The joint position vector 
$\bm{\theta} = [\bm{\theta}_{t}\trans \; \; \bm{\theta}_{a}\trans]\trans \in \R^m$ consists of the joint positions of both the positioning table $\bm{\theta}_{t} \in \R^{m-n}$ and manipulator $\bm{\theta}_{a} \in \R^{n}$.

Still, the deposition direction has to be kept aligned to a desired direction (e.g. gravity) represented in an inertial frame (e.g. $\F_{ab}$). Defining the forward kinematics considering only the manipulator as:
\begin{align}
    \bm{x}_{a} &= \begin{bmatrix}
    \bm{p}_a\\q_a
    \end{bmatrix}=k_a(\bm{\theta}_{a}) \label{eq:defFwdKinArm}
\end{align}
where $\bm{x}_a$ is the TCP pose, $\bm{p}_a \in \R^3$ is the position, $q_a  = [\eta_a \; \; \bm{\epsilon}_a\trans]\trans$ is the unit quaternion representing the orientation of the TCP with respect to $\mathcal{F}_{ab}$. 

Thus, the second task is defined as an orientation tracking problem where the TCP orientation $z$-axis has to follow a desired orientation $z$-axis defined in $\F_{ab}$.

%%%%%%%%%%%%%%%%%%%%%
%%%%%%%%%%%%%%%%%%%%
%%%%%%%%%%%%%%%%%%%%%
\section{Differential kinematics model}
In order to design the control algorithm, first the differential kinematic model of both the manipulator and the complete redundant robotic system is derived. The redundant system Jacobian $J(\bm{\theta}) \in \R^{6 \times m}$ maps the joint velocities $\dot{\bm{\theta}}$ to TCP velocities vector $\bm{v}$ as:
\begin{align}
    \bm{v} &= \begin{bmatrix}
    \bm{\dot{p}} \\ \bm{\omega}
    \end{bmatrix}= 
    J(\bm{\theta})\ \dot{\bm{\theta}}
    \label{eq:defDiffKin}
\end{align}
where the TCP velocities vector $\bm{v}$ is composed of its linear $\bm{\dot{p}}\in \R^3$ and angular $\bm{\omega} \in \R^3$ velocities with respect to $\mathcal{F}_d$ represented in $\mathcal{F}_d$. This relationship can also be expressed by partitioning the Jacobian matrix $J(\bm{\theta})$ into column blocks corresponding to each robot joints:
\begin{align}
    \bm{v} &= 
    \begin{bmatrix}
    J_1(\bm{\theta}) && J_2(\bm{\theta})
    \end{bmatrix} \begin{bmatrix} \dot{\bm{\theta}}_t\\\dot{\bm{\theta}}_a \end{bmatrix}
    %\label{eq:defDiffKin2_matrix}\\
    %\bm{v} 
    = J_1(\bm{\theta}) \dot{\bm{\theta}}_t + J_2(\bm{\theta}) \dot{\bm{\theta}}_a
    \label{eq:defDiffKin2}
\end{align}
where $\dot{\bm{\theta}}_t$ and $\dot{\bm{\theta}}_a$ are the joint velocity vectors of the table and the manipulator, respectively. $J_1(\bm{\theta}) \!\in\! \R^{6\times m-n}$ and $J_2(\bm{\theta})\!\in\! \R^{6\times n}$ are the Jacobian columns that map the table and the manipulator joint velocities to task space velocities.

Furthermore, the relationship between the quaternion derivative $\dot{q}$ and angular velocity $\bm{\omega}$ is given by the Representation Jacobian $B(q)$:
\begin{equation}
\label{eq:jacobianrep}
    \dot{q} = B(q) \bm{\omega} = \frac{1}{2}
    \begin{bmatrix}
        -\bm{\epsilon}\trans \\ \eta I - \hat{\bm{\epsilon}}
    \end{bmatrix}\ \bm{\omega}
\end{equation}

Note that the TCP velocities can be represented with respect to the TCP frame $\F_t$, denominated body coordinates \cite{murray1994}, this way, the mapping between the joint velocities and the TCP velocities is given by:
\begin{align}
    \bm{v}^b &= J^b \bm{\dot{\theta}} =
    \begin{bmatrix}
    J^b_1(\bm{\theta}) && J^b_2(\bm{\theta}_a)
    \end{bmatrix} \begin{bmatrix} \dot{\bm{\theta}}_t\\\dot{\bm{\theta}}_a \end{bmatrix}\label{eq:defDiffKin_body}\\
    \bm{v}^b &=
    \Omega\trans \begin{bmatrix}J_1(\bm{\theta}) && J_2(\bm{\theta})
    \end{bmatrix} \begin{bmatrix} \dot{\bm{\theta}}_t\\\dot{\bm{\theta}}_a \end{bmatrix}
    \label{eq:defDiffKin_body_matrix}
\end{align}
Where $\bm{v}^b$ is the TCP velocities vector with respect to $\F_d$ and represented in $\F_t$, $J^b$ denotes the full system Jacobian in frame $\F_t$ \cite{siciliano2011robotics} and $\Omega$ is the mapping between the two representations written as
\begin{equation}
    \Omega = \begin{bmatrix}
    R && 0\\0 &&R
    \end{bmatrix}
    \label{eq:ext_rotation}
\end{equation}

Similarly, the forward and the differential kinematics, considering only the manipulator is given by:
\begin{align}
    \bm{v}^b_{a} &= \begin{bmatrix}
    \bm{\dot{p}}^b_a\\\bm{\omega}^b_a
    \end{bmatrix}=J^b_{a}(\bm{\theta}_{a})\ \dot{\bm{\theta}}_{a}
    \label{eq:defDiffKinArm}
\end{align}
where $J^b_a(\bm{\theta}_a) \in \R^{6 \times n}$ is the manipulator body Jacobian \cite{murray1994} mapping the manipulator joint velocities $\dot{\bm{\theta}}_a$ to TCP linear and angular velocities ($\dot{\bm{p}}^b_a \in \R^3$ and $\bm{\omega}^b_a \in \R^3$, respectively) with respect to $\mathcal{F}_{ab}$ and represented in the TCP frame $\mathcal{F}_t$.

The quaternion propagation may also be written in function of the body angular velocity $\bm{\omega}^b_a$ \cite{Wen1991} as the representation Jacobian in body coordinates as:
\begin{equation}
\label{eq:jacobianrep_body}
    \dot{q}_a = B_b(q_a)\ \bm{\omega}^b_a = \frac{1}{2} \begin{bmatrix}
        -\bm{\epsilon}_a\trans \\ \eta_a I + \hat{\bm{\epsilon}}_a
    \end{bmatrix} \bm{\omega}^b_a
\end{equation}
Note that the manipulator TCP velocity is expressed in body coordinates which is related with \refeq{defDiffKin_body} as:
\begin{equation}
    J^b_{a}(\bm{\theta}_{a}) = J^b_2 (\bm{\theta}_{a}) = \Omega^T\,J_{2}(\bm{\theta})
    \label{eq:jacobianrelation}
\end{equation}%
This relation is useful to reduce computational cost, as only the full system forward kinematics and Jacobian needs to be calculated. From now on, the joint position dependencies of each Jacobian shall be omitted.

%%%%%%%%%%%%%%
%%%%%%%%%%%%%%
%%%%%%%%%%%%%%
\section{Task space augmentation}
\label{sec:augmented_jacobian}
The problem of accomplishing a main trajectory tracking task, where all operational space degrees of motion are needed, while maintaining the TCP aligned with a certain direction with respect to a inertial frame can be solved by the task space augmentation \cite{Egeland1987}.

Letting $\bm{x}_s=[\bm{x}_{s,1}\ \hdots \ \bm{x}_{s,p}]\trans$ be the $p$ constraint additional tasks vector to be fulfilled besides a main task defined by $\bm{x}_m$, where $\bm{x}_{m}=k_m(\bm{\theta})$ is the direct kinematics of the first task and $\bm{x}_{s,i}=k_i(\bm{\theta})$ is the direct kinematics of the $i$-th constraint-task. The differential kinematics for the main task is given by the mapping between the joint velocities and the first task velocities $\bm{v}_{m}$:
\begin{equation}
    \label{eq:main_task_mapping}
    \bm{v}_m=J_m(\bm{\theta})\dot{\bm{\theta}}
\end{equation}
where $\dot{\bm{\theta}}$ are the joint velocities and $J(\bm{\theta})$ is the first task Jacobian mapping. Letting the differential kinematics for each additional task be the mapping between the joint velocities and the task velocities $\bm{v}_{s,i}$ by:
\begin{equation}
    \label{eq:tasks_mapping}
    \bm{v}_{s,i}=J_{s,i}(\bm{\theta})\dot{\bm{\theta}}
\end{equation}
where $J_{s,i}(\bm{\theta})$ is the $i$th constraint-task Jacobian. The complete constraint differential kinematics can be written by stacking the tasks in a single relation:
\begin{equation}
    \label{eq:constraint_tasks}
    \bm{v}_s=J_s(\bm{\theta})\dot{\bm{\theta}}
\end{equation}
where $J_s=[J_{s,1}\trans,~J_{s,2}\trans,~\hdots,~J_{s,p}\trans]\trans$ and $\bm{v}_s=[\bm{v}_{s,1}\trans,~\bm{v}_{s,2}\trans,~\hdots,~\bm{v}_{s,p}\trans]\trans$. 
Stacking \refeq{main_task_mapping} and \refeq{constraint_tasks} yields a single diferential relation:
\begin{equation}
    \begin{bmatrix}
       \bm{v}_m \\ \bm{v}_s
    \end{bmatrix}=
    \begin{bmatrix}
        J_m(\bm{\theta})\\ J_s(\bm{\theta})
    \end{bmatrix}\dot{\bm{\theta}}
    =J_A(\bm{\theta})\dot{\bm{\theta}}    
    \label{eq:augmented_jacobian_base}
\end{equation}
where $J_A(\bm{\theta})$ is called augmented Jacobian. From now on, the joint position dependencies of each Jacobian shall be omitted.

\subsection{Augmented tasks for WAAM redundant systems}
Now, considering a kinematic chain composed by two robots as described in \reffig{cell-drawing}, assuming that the operational space is described by 6-DOF of movement and establishing the main task as the end-effector motions with respect to the deposition frame $\F_d$ \refeq{defDiffKin}.

Letting $\bm{\omega}_s \in \R^r$ be $r$ degrees of the TCP motion considering only the manipulator differential kinematics \refeq{defDiffKinArm} with respect to $\F_{ab}$. This constraint-task can be written by using a selection matrix  $H \!\in\! \R^{r\times6}$ and defined the additional task as:%
\begin{align}
    \bm{\omega}_s &= H \bm{v}^b_{a}\notag\\
    &= H J^b_{a} \dot{\bm{\theta}}_a = H \Omega\trans J_2 \dot{\bm{\theta}}_a \notag\\
    &=\Lambda J_2 \dot{\bm{\theta}}_a  \in \R^r
    \label{eq:restrH_augmented}
\end{align}
where $\Lambda=H\Omega\trans \in \R^{r\times 6}$ and $H J^b_{a}\in \R_{r\times n}$\footnote{Without loss of generality, $\Lambda$ could be replaced by $H$ if the second task is represented in the inertial frame}. The task augmentation problem can be described by stacking the main task \refeq{defDiffKin2} and the additional constraint-task\refeq{restrH_augmented} as:
\begin{equation}
    \begin{bmatrix}
        \bm{v}\\\bm{\omega}_s
    \end{bmatrix}=
    \begin{bmatrix}
        J_1 && J_2\\
            0_{r\times m-n} && \Lambda J_2
    \end{bmatrix} \dot{\bm{\theta}}=
    J_A \dot{\bm{\theta}}
    \label{eq:augmented_relation}
\end{equation}
Where $J_A \in \R^{6+r\times m}$ is the augmented Jacobian for this problem. Note that the secondary task does not consider the table joint velocities, since they do not affect the end-effector motion with respect to the inertial frame $\F_{ab}$.

The kinematic model \refeq{augmented_relation} has the following property:

\begin{property}
\label{prop:bounded_augmented_jacobian}
The Jacobian $J_A(\bm{\theta})$ is bounded for all $\bm{\theta}$, that is, $||J_A(\bm{\theta})||_\infty \leq c_0$, $\forall {\theta}_i \in [0, 2\pi)$, where $c_0>0$ and $i = 1,\dots,m$.
\end{property}

For the WAAM system pictured in \reffig{cell-drawing}, composed by a $6$-DOF manipulator and a $2$-DOF, which results in a $8$-DOF redundant kinematic chain, the selection matrix \refeq{restrH_augmented} is defined in order to align the TCP with a direction described in an inertial frame, in this case $\F_{ab}$. 
Naturally, the task of aligning the TCP with a desired direction $\bm{z}^b_{d}$ with respect to the inertial frame $\F_{ab}$ can be accomplished by controlling rotations about $x$ and $y$ axes. Therefore, considering $r=2$ and $\bm{\omega}_s = [\omega_x \quad \omega_y]\trans \!\in\! \R^2$, $H$ is given by:
\begin{equation}
    H=\begin{bmatrix}
    0 && 0 && 0 && 1 && 0 && 0\\
    0 && 0 && 0 && 0 && 1 && 0
    \end{bmatrix}
    \label{eq:defH}
\end{equation}    

\subsection{Control law}
For a kinematic control law, it is assumed that both the robot manipulator and the positioning table allows joint velocity command $\bm{u}$ and have internal controllers that guarantee that $\dot{\bm{\theta}} - \bm{u} \in \mathcal{L}_2 \cap \mathcal{L}_\infty$, meaning they are sufficiently fast. This assumption is taken because most robots used in industrial applications have a closed control architecture which does not allow the user to modify the parameters of the joint control algorithm. 
Therefore, given the kinematic control assumption:
\begin{equation}
    \dot{\bm{\theta}}=\bm{u}+\bm{\eta}
    \label{eq:unmodeled_dynamics}
\end{equation} 
Where $\bm{\eta}=\{\eta_1,\eta_2,\hdots,\eta_m\}$ and $\eta_i \in \mathcal{L}_2 \cap \mathcal{L}_\infty$  is a signal representing the $i$-th joint unmodeled dynamic inherent to the internal dynamic control loop.

Substituting \refeq{unmodeled_dynamics} in \refeq{augmented_relation} yields the following control system:
\begin{equation}
    \begin{bmatrix}\bm{v}\\\bm{\omega}_s\end{bmatrix} = J_A(\bm{\theta})\ (\bm{u} +\bm{\eta})
    \label{eq:diffkinsys}
\end{equation}
   
Thus, the following control law is used:
\begin{align}
    \bm{u} = 
     J_A(\bm{\theta})^\dagger
    \begin{bmatrix}
    \bar{\bm{u}}_1\\ \bar{\bm{u}}_2 
    \end{bmatrix}
    \label{eq:control-law}
\end{align}
where $\bar{\bm{u}}_1$ is designed so that the TCP pose follows a desired trajectory with respect to $\F_d$ while $\bar{\bm{u}}_2$ is designed so the TCP $r$ degrees of freedom can follow a desired trajectory with respect to $\F_{ab}$ as long as $\bar{\bm{u}}_2$ dimension $r\leq m-n$.

In order to align the manipulator TCP $z$-axis with a desired
direction $\bm{z}^b_{d}$ (e.g. direction of gravity or any other constant desired direction described in the inertial frame) without changing the TCP pose $\bm{x}$ \refeq{fk} with respect to $\F_d$, two redundant degrees of mobility are required, as $r=2$.

The system \refeq{control-law} is valid if and only if, the following assumptions hold: 
(\textbf{A1}) the robot kinematics is perfectly known;
(\textbf{A2}) the control law $[\bar{\bm{u}}_1\trans\quad \bar{\bm{u}}_2\trans]\trans$ does not drive the robotic system to singular configurations.

For the trajectory following task, the position error $\bm{e}_p$ is defined by
$\bm{e}_p = \bm{p}_{d} - \bm{p}$
where $\bm{p}_{d}$ is the desired position of $\mathcal{F}_t$ with respect to $\mathcal{F}_d$.
The orientation error $e_q$ is defined by
$e_q = [e_{\eta} \; \; \bm{e}_{\bm{\epsilon}}\trans]\trans = q_{d} \otimes q^{-1}$
where $q_{d}$ is the desired orientation of $\mathcal{F}_t$ with respect to $\mathcal{F}_d$,
and $\otimes$ indicates the quaternion product.

The control signal $\bar{\bm{u}}_1$ is defined:
\begin{equation}
    \bar{\bm{u}}_1 = \begin{bmatrix} 
    \bm{\dot{\bm{p}}}_{d} \\ 
    \bm{\omega}_{d}
    \end{bmatrix}+\begin{bmatrix} 
    K_p\ \bm{e}_p \\ 
    K_o\ \bm{e}_{\bm{\epsilon}}
    \end{bmatrix}
    \label{eq:firstcontrolsignal}
\end{equation}
where $K_p = k_p I\!\in\!\R^{3\times3}$ and $K_o = k_o I\!\in\!\R^{3\times3}$ are the controller gains, %($k_p, k_o > 0$), 
and $\bm{\dot{p}}_{d}, \bm{\omega}_{d} \in \R^3$ are the desired linear and angular velocity, respectively.

Now considering $\bm{z}^b$, $z$-axis of $\mathcal{F}_t$, and the desired $\bm{z}^b_{d}$ direction both represented in $\mathcal{F}_{t}$, the orientation error is given by:
\begin{equation}
\label{eq:es_extended}
    e_{qs} = 
    \begin{bmatrix}
    e_{\eta s}\\\bm{e}_{\bm{\epsilon} s}
    \end{bmatrix} 
    = 
    \begin{bmatrix}
    e_{\eta s}\\
    \begin{bmatrix}e_{\epsilon_1s}\quad e_{\epsilon_2s}\quad e_{\epsilon_3s}\end{bmatrix}\trans
    \end{bmatrix} =
    \begin{bmatrix}
    \cos{(\alpha/2)}\\
    \bm{r} \sin{(\alpha/2)}
    \end{bmatrix}
\end{equation}
where $\alpha = \cos^{-1}(\bm{z}^b_{d}\cdot\bm{z}^b)$ and $\bm{r}=(\bm{z}^b \times \bm{z}^b_{d})/\lVert \bm{z}^b \times \bm{z}^b_{d}\rVert$ are the angle and axis of the equivalent rotation needed to align $\bm{z}^b$ and $\bm{z}^b_{d}$. For the special case of aligning the TCP $z$-axis with a given reference with respect to $\F_{ab}$, $\bm{r}=[r_1\ r_2\ 0]\trans$ and $\bm{e}_{\bm{\epsilon} s}=[e_{\epsilon_1s}\ e_{\epsilon_2s}\ 0]\trans=[\bm{e}_{s}\ 0]\trans$. 

The control signal $\bar{\bm{u}}_2$ to align the TCP is defined by:
\begin{equation}
    \bar{\bm{u}}_2 = \bm{\omega}_{sd} + K_s\ 
    %sin(\alpha) . 
    \bm{e}_s
    \label{eq:secondcontrolsignal}
\end{equation}
where $K_s = k_sI\!\in\!\R^{2\times2}$ is a control gain, $\bm{\omega}_{sd}$ is the TCP desired velocity with respect to its own frame $\F_{ab}$ and $\bm{e}_s\!\in\!\R^2$ is the orientation error considering only the rotations about $x$ and $y$ axis. 

Regarding the second task, substituting \refeq{es_extended} in \refeq{jacobianrep_body} yields: 
\begin{align}
    \dot{\bm{e}}_{qs} &= B_B(e_{qs}) \tilde{\bm{\omega}}^b_a = \frac{1}{2}
    \begin{bmatrix}
        -\bm{e}_{\bm{\epsilon} s}\trans \\ e_{\eta s} I + \hat{\bm{e}}_{\bm{\epsilon} s}
    \end{bmatrix}\ \begin{bmatrix}\tilde{\bm{\omega}}_s \\ \tilde{\omega}_z^b\end{bmatrix}
    \label{eq:propagation_constrained}
\end{align}
where $\tilde{\bm{\omega}}_s=[\tilde{\omega}^b_x \quad \tilde{\omega}^b_y]$ and, as $e_{\epsilon_3s}=0$, the quaternion error propagation is given by:
\begin{align}
    \dot{\bm{e}}_{qs} &=\begin{bmatrix}\dot{e}_{\eta s}\\\dot{\bm{e}}_{s} \\ \dot{e}_{\epsilon_3s}\end{bmatrix} = 
    \frac{1}{2}\begin{bmatrix}
    -\bm{e}_{s}\trans \bm{\tilde{\omega}}_s \\ e_{\eta s}\tilde{\bm{\omega}}_s +\begin{bmatrix}
        e_{\epsilon_2 s} \\-e_{\epsilon_1 s}
    \end{bmatrix}\tilde{\omega}^b_z \\[0.5cm]
    \begin{bmatrix}
    -e_{\epsilon_2s} & e_{\epsilon_1s}
    \end{bmatrix}\tilde{\bm{\omega}}_s + e_{\eta s}\tilde{\omega}^b_z
    \end{bmatrix}  
    \label{eq:qserrorpropagation}
\end{align}

Thus, considering \refeq{control-law}-\refeq{secondcontrolsignal}, the closed-loop error dynamic is given by:

\begin{equation}
    \begin{bmatrix} 
    \dot{\bm{e}}_p \\ 
    \tilde{\bm{\omega}} \\\tilde{\bm{\omega}_s}
    \end{bmatrix} = \begin{bmatrix} 
    \dot{\bm{p}}_{d} - \dot{\bm{p}} \\ 
    \bm{\omega}_{d} - \bm{\omega} \\
    \bm{\omega}_{sd} - \bm{\omega}_s
    \end{bmatrix} =\begin{bmatrix} 
    - K_p\ \bm{e}_p \\ 
    - K_o\ \bm{e}_{\bm{\epsilon}}\\
    - K_s\ \bm{e}_{s}
    \end{bmatrix} - J_A(\bm{\theta}) \bm{\eta}
    \label{eq:errordynamics}
\end{equation}
And $\bm{e}_p \to 0$, $\bm{e}_{\bm{\epsilon}} \to 0$ and $\bm{e}_{s} \to 0$ if $K_p, K_o, K_s > 0$.

Considering the property \refprop{bounded_augmented_jacobian}, $J_A(\bm{\theta}) \!\in\! \cal{L}_\infty$, and given that
$\bm{\eta} \!\in\! \mathcal{L}_{2} \cap \cal{L}_\infty$, then  $J_A(\bm{\theta})\bm{\eta} \!\in\! \mathcal{L}_{2} \cap \cal{L}_\infty$ \cite{khalil02}.
Thus, it can be considered that there exists a positive constant $L_m$ such that
$\int_0^t \bm{\eta}\trans J_A\trans J_A \bm{\eta} d\tau \leq L_m ~ \forall ~ t$. 
Therefore, the following theorem with properties of the closed loop system stability is presented:

\begin{theorem}
\label{theo:extended_jacobian}
    Consider the closed-loop system described by \refeq{errordynamics}
    Assume that the reference signal 
$p_{d}$ is piecewise continuous and uniformly bounded in norm, and $q_{d}$ the is the unit quaternion representation of $R_{d} \in SO(3)$.
Then, under assumptions (\textbf{A1}) and (\textbf{A2}), and considering 
$\bm{\eta} \in \mathcal{L}_{2} \cap \cal{L}_\infty$, 
the following properties hold:
(i) all signals of the closed-loop system are uniformly bounded;
(ii) $\lim_{t \rightarrow \infty} \bm{e}_{p}=0$, $\lim_{t \rightarrow \infty} \bm{e}_{\bm{\epsilon}} =0$,
$\lim_{t \rightarrow \infty} e_{\eta}(t) = \pm1$, $\lim_{t \rightarrow \infty} \bm{e}_{s} =0$ and
$\lim_{t \rightarrow \infty} e_{\eta s}(t) = \pm1$. 
\end{theorem}

\begin{proof}
    See the appendix \ref{sec:appendix}.
\end{proof}
%%%%%%%%%%%%%%
%%%%%%%%%%%%%%
%%%%%%%%%%%%%%
\section{Constrained kinematics control}
\label{sec:contrained-kinematics}

Another option for the coordinated motion is based in the constrained Jacobian formulation \cite{phamIfac2014,Pham2015}. The kinematic chain is again composed by two robots as described in \reffig{cell-drawing} and it is assumed that the operational space is described by 6-DOF of movement. The objective is to describe the redundant system differential kinematics in order to relate extra degrees of motion, called constrained variables, with respect to the $\F_{ab}$ without affecting the TCP trajectory tracking with respect to the deposition frame $\F_d$.

\subsection{Constrained differential model}
Now, considering that $r$ degrees of freedom of the manipulator differential kinematics \refeq{defDiffKinArm} as another task. One can write a reduced differential kinematics relation as in \refeq{restrH_augmented} and assuming that $\Lambda$ is composed by orthonormal rows, then it has a right pseudo-inverse $\Lambda^{\dagger}=\Lambda\trans~\in~\R^{6\times r}$ ($\Lambda \Lambda^{\dagger} = I_r \in \R^{r\times r}$), \refeq{restrH_augmented} can be written as:
\begin{align}
    &\Lambda J_2 \dot{\bm{\theta}}_a = \bm{\omega}_s \nonumber \\
    &\Lambda(J_2 \dot{\bm{\theta}}_a- \Lambda\trans \bm{\omega}_s) = \bm{0} \in \R^r
\label{eq:restrH_complete}
\end{align}

Clearly $J_2 \dot{\bm{\theta}}_a- \Lambda\trans \bm{\omega}_s$ lies in the null space of $\Lambda$ such that $dim(N(\Lambda))={6-r}$, thus, there exists an auxiliary control vector $\bar{\bm{v}}\!\in\! \R^{6-r}$ that satisfies:
\begin{align}
    J_2 \dot{\bm{\theta}}_a- \Lambda\trans \bm{\omega}_s &=\Lambda^\#\bar{\bm{v}}\\
    J_2 \dot{\bm{\theta}}_a &= \Lambda^\#\bar{\bm{v}} + \Lambda\trans \bm{\omega}_s
    \label{eq:H_null}
\end{align}
where $\Lambda^\# \!\in\! \R^{6\times 6-r}$ is the null space projector of $\Lambda$ ($\Lambda \Lambda^{\#} = \bm{0}$).

Substituting \refeq{H_null} in \refeq{defDiffKin2}, one has:

\begin{align}
    \bm{v} &=  J_1\dot{\bm{\theta}}_t + \Lambda^\#\bar{\bm{v}} + \Lambda\trans \bm{\omega}_s\label{eq:vrelation_extend}\\
     %\bm{v} 
     &=  J_c
    \begin{bmatrix} \dot{\bm{\theta}}_t\\ \bar{\bm{v}}\end{bmatrix} +\Lambda\trans \bm{\omega}_s
    \label{eq:vrelation_short}
\end{align}
where $J_c = \begin{bmatrix}J_1 && \Lambda^\#\end{bmatrix}\!\in\! \R^{6\times m-n+6-r}$ is the constrained Jacobian \cite{phamIfac2014}. The control problem can be decoupled between the table and manipulator joint velocities signals by assuming that $J_c$ has a right pseudo inverse, which yields:
\begin{equation}
    \label{eq:first_control_problem}
    \begin{bmatrix} \dot{\bm{\theta}}_t\\ \bar{\bm{v}}\end{bmatrix}=J_c^\dagger (\bm{v} -\Lambda\trans \bm{\omega}_s)
\end{equation}

Now, assuming that $J_2$ is a full rank matrix and has a right pseudo-inverse, \refeq{H_null} can be written as:
\begin{align}
    \dot{\bm{\theta}_{a}}&={J_2}^\dagger(\Lambda^\#\bar{\bm{v}} + \Lambda\trans \bm{\omega}_s)
    \label{eq:constrainedThetaA}
\end{align}

\subsection{Equivalence with the Augmented Jacobian}

In order to extend the velocity reference vector to comprise both the TCP velocities ($\bm{v}$) with respect to $\F_d$ and the constrained velocities vector ($\bm{\omega}_s$) with respect to the inertial frame $\F_{ab}$, the relation \refeq{vrelation_extend}, can be written as:
\begin{equation}
    \begin{bmatrix}
        \bm{v} \\ \bm{\omega}_s
    \end{bmatrix} =
    \begin{bmatrix}
        J_1 && \Lambda^\# && \Lambda\trans\\ 0 && 0 && I
    \end{bmatrix} 
    \begin{bmatrix}
        \dot{\bm{\theta}}_t \\ \bar{\bm{v}}\\ \bm{\omega}_s
    \end{bmatrix}
    \label{eq:vrelation_extended}
\end{equation}
Given that $\begin{bmatrix}
    \Lambda^\#&& \Lambda\trans
\end{bmatrix} \in \R^{6\times 6}$ consists of a full-rank orthonormal matrix, \refeq{H_null} yields:

\begin{align}
    \begin{bmatrix}\bar{\bm{v}}\\\bm{\omega}_s\end{bmatrix}=
    \begin{bmatrix}
    \Lambda^\#&& \Lambda\trans
\end{bmatrix}\trans J_2\dot{\bm{\theta}}_a=
    \begin{bmatrix}
    {\Lambda^\#}\trans \\ \Lambda
\end{bmatrix} J_2\dot{\bm{\theta}}_a
\label{eq:constr_vector_relation}
\end{align}

Substituting \refeq{constr_vector_relation} in \refeq{vrelation_extended} and, assuming that $\Lambda^\#$ is composed by orthonormal rows which implies that $\Lambda^\# {\Lambda^\#}\trans=I\in \R^{6-r}$ yields:
\begin{align}
    \begin{bmatrix}
        \bm{v} \\ \bm{\omega}_s
    \end{bmatrix} &=
    \begin{bmatrix}
        J_1 && \Lambda^\# && \Lambda\trans\\ 0 && 0 && I
    \end{bmatrix}  
    \begin{bmatrix}
        \dot{\bm{\theta}}_t \\ \begin{bmatrix}
    {\Lambda^\#}\trans\\ \Lambda
\end{bmatrix} J_2\dot{\bm{\theta}}_a
    \end{bmatrix}\label{eq:constraineddiffkinematics1}\\
    \begin{bmatrix}
        \bm{v} \\ \bm{\omega}_s
    \end{bmatrix} &=
   \begin{bmatrix}
        J_1 && J_2\\ 0 && \Lambda J_2
    \end{bmatrix} 
    \begin{bmatrix}
        \dot{\bm{\theta}}_t \\ \dot{\bm{\theta}}_a
    \end{bmatrix}
    = J_A
    \begin{bmatrix}
        \dot{\bm{\theta}}_t \\ \dot{\bm{\theta}}_a
    \end{bmatrix}
    \label{eq:constraineddiffkinematics}
\end{align}
where $J_A \!\in\! \R^{6+r\times m}$ is the same augmented Jacobian \refeq{augmented_relation} that maps the joint velocities signal to the augmented velocity vector $[\bm{v}\trans \bm{\omega}_s\trans]\trans$. 

Note that \refeq{constraineddiffkinematics1}-\refeq{constraineddiffkinematics} in addition to show the equivalence between the constrained kinematics algorithm, also proves that the additional task of the TCP alignment with a desired direction lies in the null space of the first task.
%%%%%%%%%%%%%%
%%%%%%%%%%%%%%
%%%%%%%%%%%%%%
\section{Singularity Analysis}
\label{sec:singularity}

In addition to the common singular configurations inherent to the kinematics of a manipulator, an algorithmic singularity arises for the augmented Jacobian control algorithm \refeq{augmented_jacobian_base} when $\mathcal{R}\{J_m\trans\}\cap \mathcal{R}\{J_s\trans\}\neq \{\bm{0}\}$ \cite{Chiacchio1991,handbook}, meaning that the augmented Jacobian has conflicting tasks.

For the WAAM system of \reffig{cell-drawing} and the augmented task problem defined in \refeq{augmented_relation}, this algorithmic singularity arises when $\F_d$ has its $z$-axis aligned with the inertial frame used to design the secondary task. Algebraically that means that $J$ and $\Lambda J_2$ have linearly dependent rows and the augmented Jacobian $J_A$ is not invertible. This becomes obvious in the constrained Jacobian algorithm, since this configuration leads the first column of $J_1$ to become linearly dependent of the columns of $\Lambda^\#$. 

Note that this issue also arises for the task-priority algorithm proposed in \cite{Lizarralde2022}, but in this case the projection matrix $J_sJ^\#$ can not be inverted.

Consider the augmented Jacobian $J_A\in \R^{n+r\times m}$ and it singular value decomposition:

\begin{equation}
    J_A=U\Sigma V\trans=\sum_{i=1}^k \sigma_i \bm{u}_i\bm{v}_i\trans
    \label{ew:svd}
\end{equation}
where $U$ is the orthonormal matrix of left-singular vectors $\bm{u}_i$, $V$ is the orthonormal matrix of the right-singular vectors $\bm{v}_i$ and $\Sigma$ is the diagonal matrix composed by the singular values ${\sigma}_i$ of the matrix $J_A$.

\begin{equation}
    J_A^\dagger=V\Sigma^\dagger U\trans=\sum_{i=1}^k \frac{1}{\sigma_i} \bm{v}_i\bm{u}_i\trans
    \label{ew:pinv}
\end{equation}

The singularity occurs when $k<n+r$ and the Jacobian has null singular values. A common algorithm to reduce the effects of the inversion of the ill-conditioned Jacobian is the Damped Least Square (DLS) method \cite{Nakamura1986} where the control law \refeq{control-law} is rewritten as:
\begin{align}
	\bm{u} &= J_{A}\trans\left(J_{A}\left(J_{A}\right)\trans+\kappa^2I_n\right)^{-1} \begin{bmatrix}
	    \bar{\bm{u}}_1\\\bar{\bm{u}}_2
	\end{bmatrix}
	\label{eq:dlsmanip}\\
	\bm{u} &= \left(\sum_{i=1}^k \frac{\sigma_i}{\sigma_i^2 +\kappa^2} \bm{v}_i\bm{u}_i\trans\right) \begin{bmatrix}
	    \bar{\bm{u}}_1\\\bar{\bm{u}}_2
	\end{bmatrix}
	\label{eq:dlssvd}
\end{align}%
where $\kappa$ is a damping factor that makes the pseudo-inverse better conditioned from a numerical point of view \cite{mayorga}. 

The damping factor can be constant, resulting in a numerically stable algorithm in the entire operating space \cite{Wampler}. On the other hand, a constant damping factor causes inaccurate movement. Alternatively, since the augmented Jacobian is ill-conditioned due the addition of the secondary task \refeq{restrH_augmented} and given that WAAM has a small acceptable tolerance in the torch tip alignment with orientation parameters, one can solve the singularity problem with the damped least-squares inverse with numerical
filtering \cite{chiaverini}, thus the control law \refeq{dlsmanip} can be written as:
\begin{equation}
 \bm{u} = \left[\left(\sum_{i=1}^{k-1} \frac{1}{\sigma_i}\bm{v}_i\bm{u}_i\trans \right) + \frac{\sigma_{k}}{\sigma_{k}^2 +\kappa^2}\bm{v}_{k}\bm{u}_{k}\trans \right]\begin{bmatrix}
	    \bar{\bm{u}}_1\\\bar{\bm{u}}_2
	\end{bmatrix}
	\label{eq:selectivedls}
\end{equation}

This strategy progressively reduces the velocity associated to the task responsible for the algorithm singularity, being in this case the task orthogonal to $\bm{u}_k$, while not affecting the deposition path tracking.
%%%%%%%%%%%%%%
%%%%%%%%%%%%%%
%%%%%%%%%%%%%%
\section{Experimental results and analysis}
\label{sec:results}
In order to validate the proposed control scheme, the experiments of this section were performed to produce thin walls parts with WAAM. This section presents the experimental setup and the manufacturing of parts which requires the coordinated movements between a manipulator and a positioning table.

\subsection{Experimental setup}
\label{subsec:experimental-setup}
For the study case experiments, a WAAM robotic cell consisting of a 6DOF KUKA KR90 robotic arm and a 2DOF KUKA KP2 positioning table, both controlled by a KUKA KRC4 Controller with KUKA.RSI (Robot Sensor Interface) software add-on installed \cite{Arbo2020} is used. The welding torch, connected to a Fronius CMT Power Source, is attached to the robot end effector (\reffig{kr-90-cell}).
\begin{figure}[ht!]
    \centering
    \includegraphics[width=0.7\columnwidth,trim={0 0 0 0},clip]{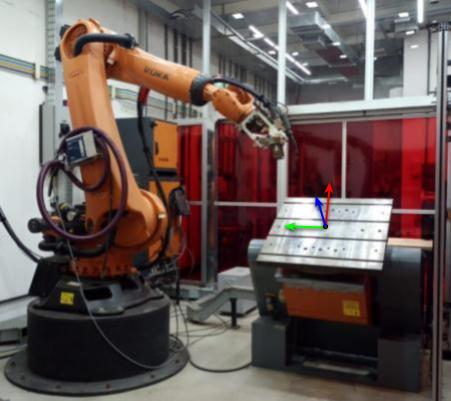}
    \caption{Robotic system setup, consisting of a KUKA KR90 robotic arm with a welding torch and a KUKA KP2 positioning table}
    \label{fig:kr-90-cell}
\end{figure}

A PC running Linux OS Ubuntu 20.04 LTS is used to run Robot Operating System (ROS) Noetic\footnote{http://wiki.ros.org/noetic} and a modified version of the ROS package \emph{kuka\_experimental} named iWAAM is used to control the robot trough velocity commands (more details about the communication setup are described in \cite{Coutinho2022}).  
The control algorithm runs in MATLAB R2021a and is set to acquire joint position information and send desired joint velocity commands at $60$\,Hz.

The trajectories are planned using the deposition frame $\F_d$ as reference, which, for better understanding of the trajectories equations, is shown in \reffig{kr-90-cell}.

Before the manufacturing experiments, single bead tests have been performed to designate the deposition parameters and travel speed, through that, it has been verified that, with the parameters of Table~\ref{tab:parameters-tab}, the travel speed of $V_{ts}=7.5$\,mm/s produces beads with an average $1.6$\,mm height while for $V_{ts}=5$\,mm/s the bead height is $2.0$\,mm.

\begin{table}[ht!]
    \caption{Deposition parameters for the experiments}
    \centering
    \begin{tabular}{|c|p{3.5cm}|}
    \hline
    Deposition mode         & CMT\\ \hline
    Average arc current                   & 110.4 A \\ \hline
    Average arc voltage                   & 14.3 V \\ \hline
    Average wire feed speed                   & 3.0 m/min \\ \hline
    Power                  & 30\% \\ \hline
    Arc length correction    & 25.0\% \\ \hline
    Dynamic correction      & 2.5\% \\ \hline
    Burn back               & 100.0 ms \\ \hline
    Preflow time            & 2000 ms \\ \hline
    Crater time             & 500.0 ms \\ \hline
    Metal feed wire                     & {AWS A5.28 - ER90s-b3}\quad {(Lincoln Electric - LNM 20)} \\ \hline
    Shielding Gas           & Argon 8\% CO2\\\hline
    Travel speed ($V_{ts}$)   &   7.5~mm/s (Inclined thin-wall) 5~mm/s (Curved thin-wall) 5~mm/s (Intake funnel)\\\hline
    Base metal   &  Carbon steel plate (ASTM A36 standard specification)\\
    \hline
    \end{tabular}
    \label{tab:parameters-tab}
\end{table}

Note that, for all parts reported in this section the strategy of alternate the arc ignition point and change the deposition direction between odd and even layers is used in order to limit deformations of the thin walls in the surrounding areas where the arc was struck, specifically at the end and beginning of each layer.

\subsection{Inclined thin-wall}
\label{subsec:inclined-thin-wall}
To introduce the concept of layer height correction using a redundant robotic system, the first experiment involves the manufacturing of an inclined thin wall. A nominal point-to-point motion timing law $s(t) \in [0,1]$ with a trapezoidal time scaling profile \cite{lynch2017modern} has been used to avoid jerky motion. 

For the sake of simplicity, the inclined thin wall has been deposited along the $y$-axis, as shown in the coordinate system in \reffig{kr-90-cell}. This configuration ensures that the motion profile is applied solely to movements along the $y$-axis, while the $x$ and $z$ axis references handle layer height adjustments. The position trajectory for this experiment is defined as:
\begin{equation}
\label{eq:traj_thin_wall}
   \bm{p}_d=\begin{bmatrix}
       \sin(\gamma)~n_l~l_h+o_x\\l_l s(t)+o_y \\\cos(\gamma)~n_l~l_h 
   \end{bmatrix}
\end{equation}
where $n_l$ is the number of layers, $l_h$ is the layer height (assumed to be constant), $o_x$ and $o_y$ is the trajectory offset in the $xy$ plane from $\F_d$ origin, $\gamma$ is the inclination angle of the thin wall to be deposited and $l_l$ is its length.

For this task, the torch orientation is set  to remain  pointing towards gravity while having an inclination $\gamma$ with the workpiece normal, which in this case is a rotation around the $y$-axis of $\F_d$. The desired orientation is then written by:
\begin{equation}
\label{eq:traj_ori_thin_wall}
    q_{d}=\begin{bmatrix}
    \cos(\gamma/2) & 0 & \sin(\gamma/2) & 0
    \end{bmatrix}\trans
\end{equation}

A thin wall with $\gamma=\pi/4$ ($45^\circ$) inclination, $150$\,mm length and $30$\,mm height has been manufactured and is shown in \reffig{thin-wall}. For this experiment, each layer was deposited in the opposite direction of the previous one 
to improve the part symmetry, and the final part has been manufactured with $20$ layers. The deposition was performed with a travel speed $V_{ts}=7.5$\,mm/s with the acceleration and deceleration time $t_a=0.1$\,s.

The position and orientation (quaternion vector part) errors for three layers ($5^{th}$, $10^{th}$ and $15^{th}$) are shown in \reffig{inclined-wall-all-errors}, recalling that the quaternion error ${e}_{q}= [1\quad 0\quad 0\quad 0]\trans$ means that the frames are aligned.
\begin{figure}[ht!]
    \centering
    \includegraphics[width=0.49\textwidth,trim={0 0 0 0},clip]{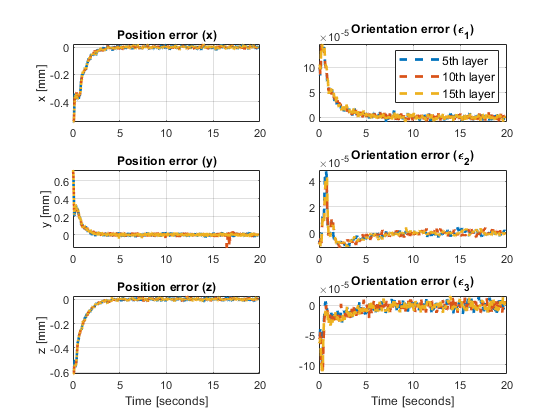}
    \caption{Position and orientation trajectory errors of the $5^{th}$, $10^{th}$ and $15^{th}$ layers of the inclined wall}
    \label{fig:inclined-wall-all-errors}
\end{figure}

For both errors in \reffig{inclined-wall-all-errors} the tracking behavior show satisfactory results, as the error converges to zero. The small tracking error at the beginning of the trajectory is expected, as the controlled system has been implemented with a slew rate acceleration saturation to avoid torque errors in the robot controller. Overall, these errors are small enough to not affect the part final geometry.

The final produced part is shown in \reffig{thin-wall}. Notice that in addition to $20$ layers of the inclined wall, a single bead has been deposited first to be used as support and prevent molten metal from dripping.

\begin{figure}[ht!]
    \centering
    \includegraphics[width=0.32\columnwidth,trim={5cm 15cm 10cm 15cm},clip]{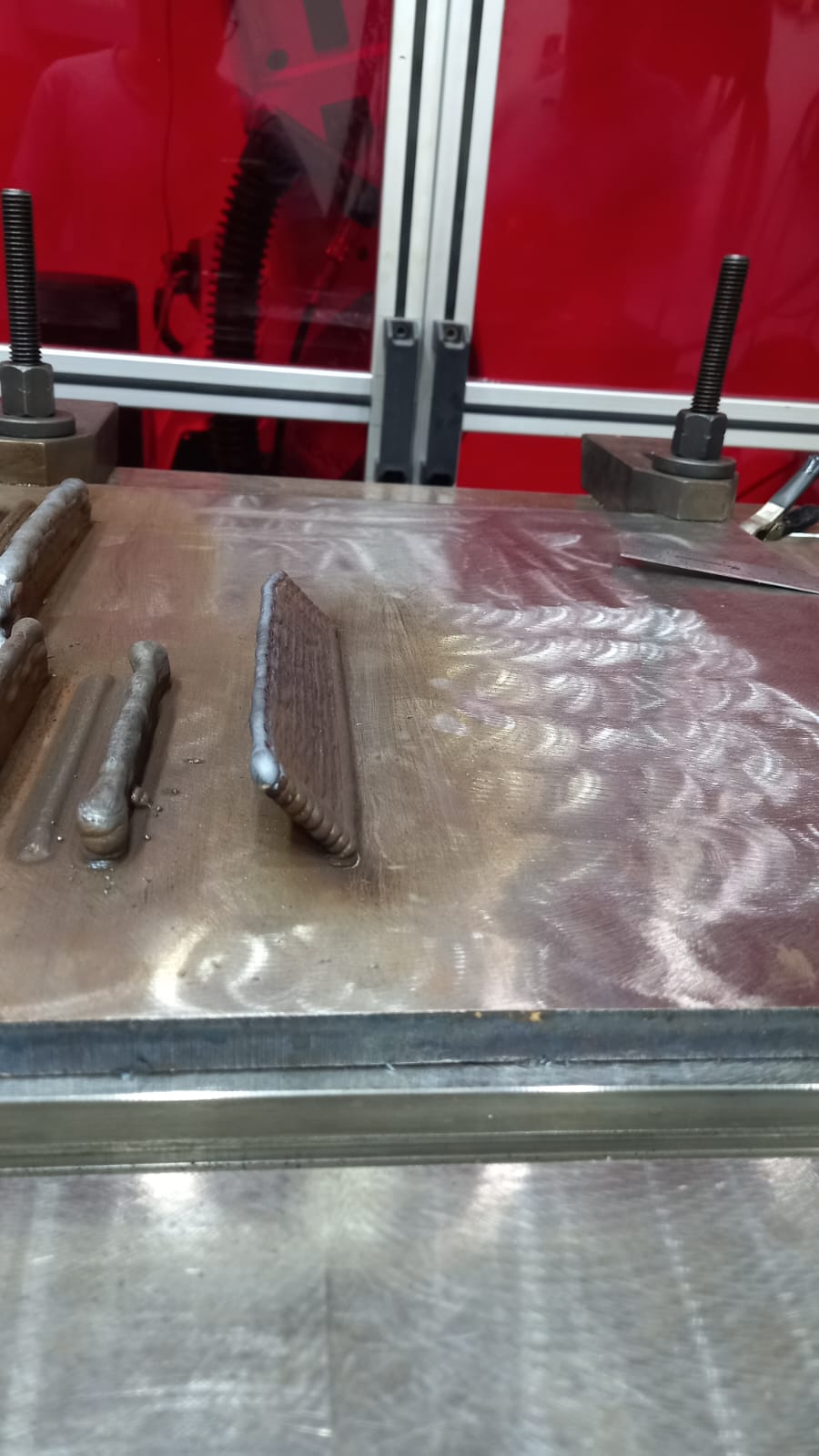}
    \caption{Final inclined wall produced part}
    \label{fig:thin-wall}
\end{figure}

\subsection{Curved thin-wall}
\label{subsec:curved-thin-wall}
A curved wall has been manufactured to attest the practicality of the proposed approach for curved geometry parts. The part has been deposited following the scheme of \reffig{curved-wall-comp}(a) divided into three sections: the first section is a thin wall with a $h_1$ height which is built as the foundation to the curved part, the second section is a $90^\circ$ curved wall, deposited by changing the torch orientation with respect to the workpiece for each layer while maintaining the torch pointing toward gravity and for the final section, another straight thin wall with a $h_2$ height is deposited on top of the curved section.
\begin{figure}[ht]
    \centering
    \subfloat[][]{\includegraphics[width=0.40\columnwidth,trim={0 0 0 0},clip]{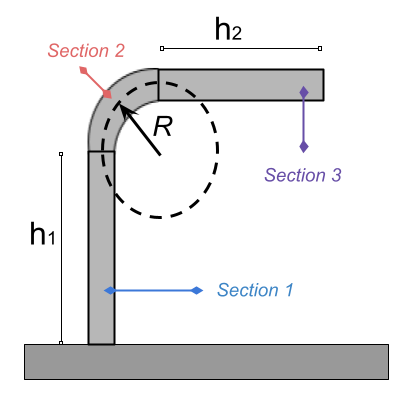}}\label{fig:curved-scheme}
    \hspace{0.2cm}
    \subfloat[][]{\includegraphics[width=0.35\columnwidth,trim={20cm 55cm 30cm 35cm},clip]{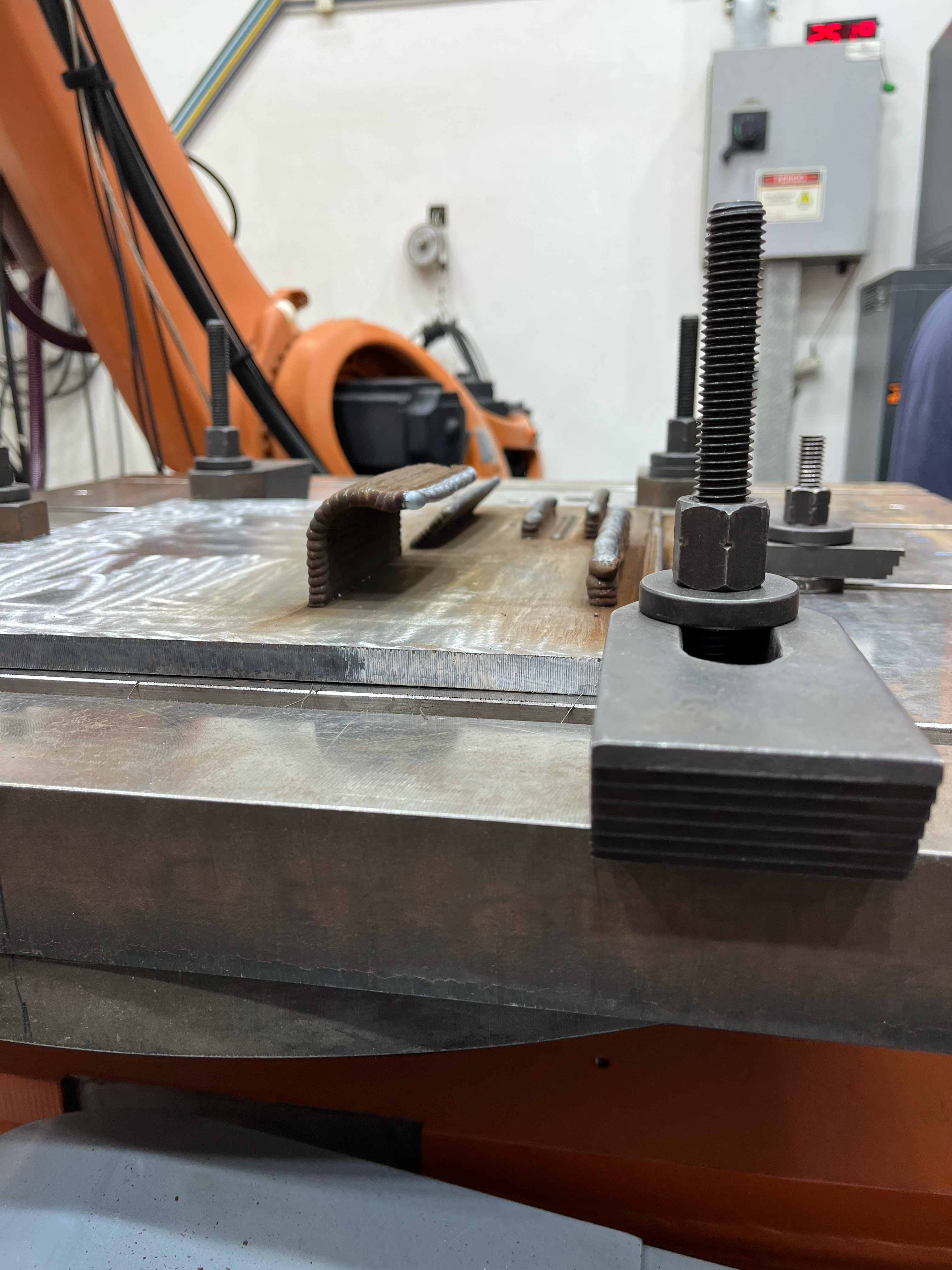}}\label{fig:curved-wall}
    \caption{(a) Curved wall manufacturing strategy scheme and (b) the produced part.}
    \label{fig:curved-wall-comp}
\end{figure}

The deposition of the first section follows the steps of the inclined thin-wall trajectory \refeq{traj_thin_wall}-\refeq{traj_ori_thin_wall}, considering $\gamma=0$, while the trajectory for the second (curved) section follows:
\begin{subequations}
\label{eq:traj_curved_wall}
\begin{align}
    \bm{p}_d&=\begin{bmatrix}
        \sin(n_l\Delta\gamma) n_l l_h+o_x \\l_l s(t)+o_y\\\cos(n_l\Delta\gamma) n_l l_h + o_z 
    \end{bmatrix}\\
    q_{d}&=\begin{bmatrix} \cos(n_l\Delta\gamma/2) & 0 & \sin(n_l\Delta\gamma/2) & 0
    \end{bmatrix}\trans
\end{align}
\end{subequations}
where $n_l$, $l_h$, $l_l$ , $o_x$ and $o_y$ follows the description of \eqref{eq:traj_thin_wall}, and $o_z$ is the offset along $\F_d$  $z$-axis (chosen to be the first section layers average height). The incremental inclination angle $\Delta\gamma$ is computed according to the curve radius $R$ and the layer height $l_h$ as:
\begin{align}
    \Delta\gamma&=\frac{\pi}{2 n_t}\label{eq:delta_alpha}\\
    n_t&=\pi R\frac{1}{2 l_h}
\end{align}
where $n_t$ is the total number of layers necessary to produce the curved section. Notice that, for the third section, $\gamma=\pi/2$ ($90^\circ$) is constant and the trajectory resembles \refeq{traj_thin_wall}, but with switched $x$ and $z$ axis trajectories.

A curved wall part manufactured with this method is show in \reffig{curved-wall-comp}(b). The final part has a total length of $l_l\approx 97$\,mm, divided into segments where $h_1=30$\,mm, $R=30$\,mm and $h_2=20$\,mm. The travel speed used for this experiment is set at $V_{ts}=5$\,mm/s with an acceleration and deceleration time of $t_a=0.1$\,s. The deposition parameters produce beads with an average height of $2.0$\,mm, resulting in a final part composed of approximately $n_t\approx 48$ layers.

The norms of the position and orientation (quaternion vector part) errors of the TCP with respect to $\F_d$ are shown in \reffig{curved-wall-mean-errors}. The chosen metric to show the errors is the root mean square (RMS) error between depositions for each time instant, as each deposition has a similar amount of data. The RMS value for each coordinate error at a given time instant is given by $e_{rms} = \sqrt{\sum_{k=1}^N e(k)^2 / N}$, where $N$ is the number of layers.  The blue area surrounding each plot curve represents the standard deviation of the trajectory RMS error at each time instant.
\begin{figure}[ht!]
    \centering
    \includegraphics[width=0.4\textwidth,trim={0 0 0 0},clip]{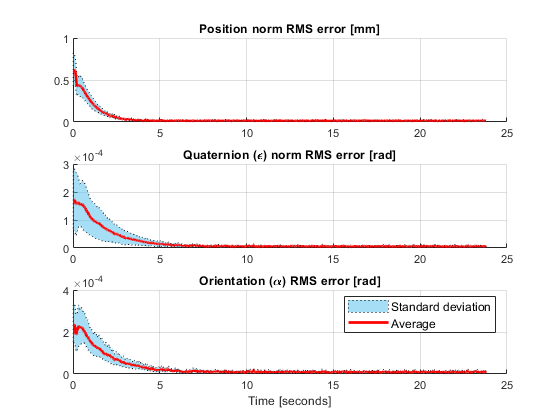}
    \caption{Position and orientation RMS error and standard deviation of the curved wall experiment.}
    \label{fig:curved-wall-mean-errors}
\end{figure}

The RMS error of the inclination angle $\alpha$ between the TCP with the desired $z$-axis in the inertial frame is also shown in \reffig{curved-wall-mean-errors}. The rapid convergence of the RMS inclination error to zero indicates that the Tool Center Point (TCP) successfully follows the commanded orientation $\bm{z}^b_{d}$ during the deposition.

It is worth to point out that, during the curved thin wall manufacturing, some adjustments were made regarding the torch orientation with respect to gravity. This was due to the fact that the torch could not satisfy the desired orientation without colliding with some of the screws holding the substrate to the base plate.

Although this issue was unintentional, the proposed framework allowed for an easy adjustment of the torch direction relative to gravity, highlighting another advantage of the method. To address this issue, the torch orientation with respect to $\F_{ab}$ has been changed to rotate $\pi/18$ ($10^\circ$) around $y$ axis, resulting in $z_{d}\approx[-0.17\quad 0\quad -0.98]\trans$.

\subsection{Intake funnel}
The final experiment consist of an intake funnel manufacturing. This is a common industrial piece to control airflow between piping segments and has a geometry that requires the two-axis movement of the workpiece.

The designed part is shown in a point cloud format in \reffig{intake-funnel-design}(a) and the part can be manufactured following the deposition scheme of \reffig{intake-funnel-design}(b). The deposition process is divided in two steps, the first one is the cylinder manufacturing while the second one refers to the bell-mouth section of the intake funnel part. Thus, for this part, the deposition process uses two different trajectories equations.
\begin{figure}[ht!]
    \centering
    \subfloat[]{\includegraphics[width=0.49\columnwidth,trim={0 0 0 0},clip]{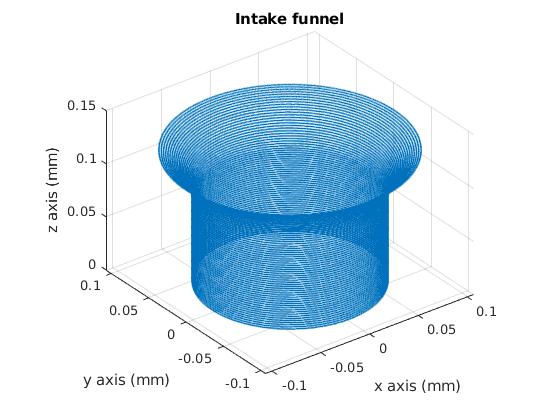}}\label{fig:intake-funnel-matlab}
    \subfloat[]{\includegraphics[width=0.49\columnwidth,trim={0 0 0 0},clip]{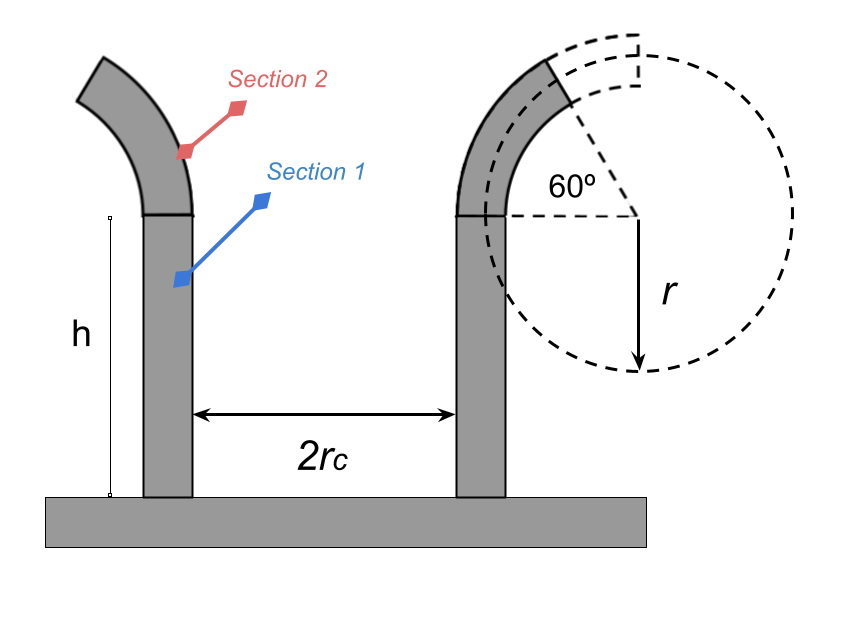}}\label{fig:intake-funnel-scheme}
    \caption{Intake funnel (a) model and (b) deposition scheme.}
    \label{fig:intake-funnel-design}
\end{figure}

The travel speed of $V_{ts}=5$\,mm/s produces layers with average height of $2.0$\,mm. The planned part consists of a cylinder with $r_c=80$\,mm radius and $h\approx 90$\,mm height (which requires the deposition of $45$ layers). For the second section has a $\pi/3$ ($60^\circ$) opening flare with a radius of $r=20$\,mm, requiring the deposition of $21$ layers as each layer height is planned to have $l_h=2$\,mm.

\subsubsection{\bf \textit{Cylinder deposition}}
The first deposition step to produce the cylinder follows the circular trajectory defined by:
\begin{subequations}
\begin{align}
    \bm{p}_d&=\begin{bmatrix}
        r_c\sin(\omega~t + \Delta\theta) + o_x \\-r_c\cos(\omega~t + \Delta\theta) + o_y\\n_l\ l_h
    \end{bmatrix}\\
    q_{d}&=\begin{bmatrix} -1 & 0 & 0 & 0
    \end{bmatrix}\trans
\end{align}
\end{subequations}
where $\omega = V_{ts}/r_c$ is the angular frequency of the circular trajectory in function of the travel speed $ V_{ts}=5$\,mm/s obtained by single bead tests, $r_c$ is the cylinder radius, $n_l$ is the layer number, $l_h$ is the layer height and $o_x$ and $o_y$ are the $x$ and $y$ axes position offsets which determine the cylinder center. The increment $\Delta\theta=n_l~\pi/36$ is chosen in order to add a phase lag for each new layer, i.e. to start the deposition $\pi/36$ ($5^\circ$) from the point where the previous layer started, this is a common strategy for closed shapes manufacturing to avoid material accumulation in the sewing points. Note that the given $q_d$ considers the torch frame $\F_t$ $z$-axis to be anti-parallel with $\F_d$ $z$-axis, i.e. normal to the deposition surface.

The deposition trajectory of some deposited layers for the cylinder section is shown in \reffig{cylinder-position-trajectory}. The deposition had an operational failure at the beginning of the $33^{rd}$ layer (lack of gas), which caused the deposition to stop soon after it has started. For this reason, a small adjustment was required justifying the phase lag between the $32^{nd}$ and $33^{rd}$ layers was increased.

\begin{figure}[ht!]
    \centering
    \includegraphics[width=0.42\textwidth,trim={0 0 0 0},clip]{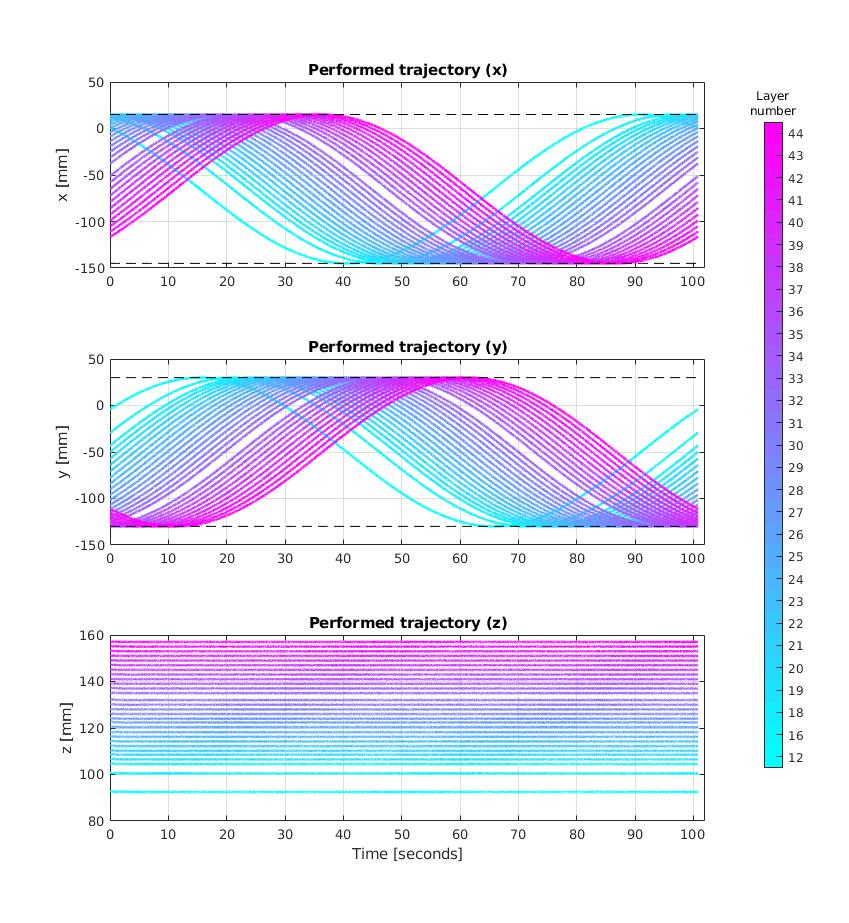}
    \caption{Intake funnel cylinder section trajectories}
    \label{fig:cylinder-position-trajectory}
\end{figure}

The position and orientation RMS error are shown in \reffig{cylinder-mean-errors}. Note that \reffig{cylinder-mean-errors} also shows the misalignment $\alpha$ between $z^b$ and $z^b_{d}$ RMS error. All errors are sufficiently small, signifying good trajectory tracking throughout the deposition.
\begin{figure}[ht!]
    \centering
    \includegraphics[width=0.4\textwidth,trim={0 0 0 0},clip]{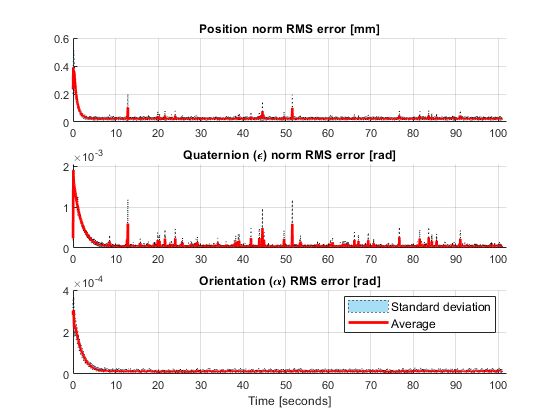}
    \caption{Position and orientation RMS errors and standard deviation for the cylinder section of the intake funnel}
    \label{fig:cylinder-mean-errors}
\end{figure}

\subsubsection{\bf \textit{Bell-mouth deposition}}
The second deposition step of the intake funnel manufacturing regards its curved section. This section consists of increasing size of the previous cylinder trajectory radius while the torch angle change for each layer, in a similar procedure made in \refsubsec{curved-thin-wall}.

The trajectory for this step is given by:
\begin{subequations}
\begin{align}
    r_{cn} &= r_c+(1-\cos(n_{l2}\Delta\gamma))r \\
    \bm{p}_d &= \begin{bmatrix}
        r_{cn}\sin(\omega~t + \Delta\theta_2) + o_x \\-r_{cn}\cos(\omega~t+\Delta\theta_2) + o_y\\ r \sin( n_{l2}\Delta\gamma)+ o_z
    \end{bmatrix}\\
    q_{d}&=\begin{bmatrix} \cos(n_{l2}\Delta\gamma/2) & 0 & \sin(n_{l2}\Delta\gamma/2) & 0
    \end{bmatrix}\trans
\end{align}
\end{subequations}
where $\Delta\theta_2=n_{l2}~\pi/36$ is the phase lag for this section, $n_{l2}$ is the layer number for this deposition step, not considering the deposited layers of the cylinder, $r_{cn}$ is the radius of the circular trajectory which changes for each new layer, $o_z$ is the offset along $z$-axis of $\F_d$. For this part, the total deposited cylinder height and $\Delta\gamma$ follows is calculated as in \refeq{delta_alpha}, but for a curvature of $\pi/3$ ($60^\circ$), i.e:
\begin{subequations}
    \label{eq:delta_alpha_intake_funnel}
\begin{align}
    \Delta\gamma&=\frac{\pi/3}{n_t}\\
    n_t&=r\frac{\pi}{3}\frac{1}{l_h}
\end{align}
\end{subequations}
where $n_t=21$ is the total number of layers necessary to produce the bell-mouth section. 

The performed $x$ and $y$ position trajectories for the bell-mouth section is shown in \reffig{intake-funnel-curved-trajectory}. Notice that the amplitude of the sinusoidal trajectory for the $x$ and $y$ axis increase with the number of layers, as it is dependent of the radius $r_{cn}$.
\begin{figure}[htb]
    \centering
    \includegraphics[width=0.5\textwidth,trim={0 0 0 0},clip]{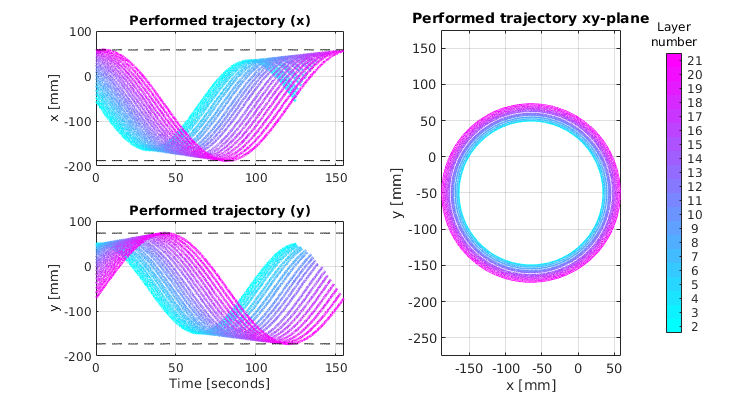}
    \caption{Intake funnel bell-mouth section $x$ and $y$ trajectories.}
    \label{fig:intake-funnel-curved-trajectory}
\end{figure}

Adjustments have been required regarding the $z$-axis trajectory and the $r_{cn}$ correction due to the wrong calibration of the positioning table and consequently the deposition frame $\F_d$ as well causing the deposition surface to misalignment with the torch. Another issue encountered was that the CTWD has not been included in the kinematic model of the robot which has also caused the torch to be too close to the deposition surface in the final layers. These adjustments have been made by correcting $r_c$ and $o_z$ to perform dry run tests before each deposition. The mean value of the performed trajectory for each layer is shown in \reffig{intake-funnel-curved-trajectory-z} alongside with the planned path.
\begin{figure}[ht!]
    \centering
    \includegraphics[width=0.4\textwidth,trim={0 0 0 0},clip]{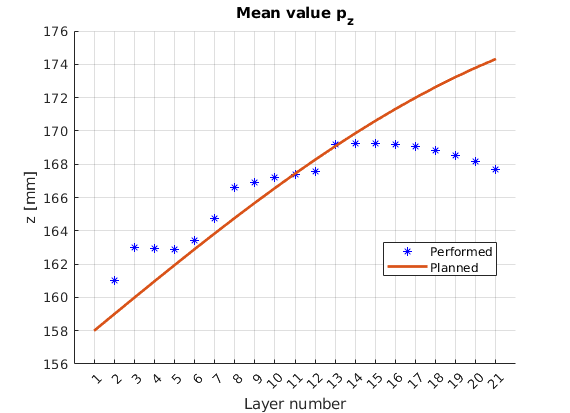}
    \caption{Intake funnel bell-mouth section $z$ trajectory mean value.}
    \label{fig:intake-funnel-curved-trajectory-z}
\end{figure}

The position and orientation RMS error are shown in \reffig{intake-funnel-curved-mean-errors}. Similar to previous experiments, errors were small enough not to compromise the deposition process.
\begin{figure}[ht!]
    \centering
    \includegraphics[width=0.4\textwidth,trim={0 0 0 0},clip]{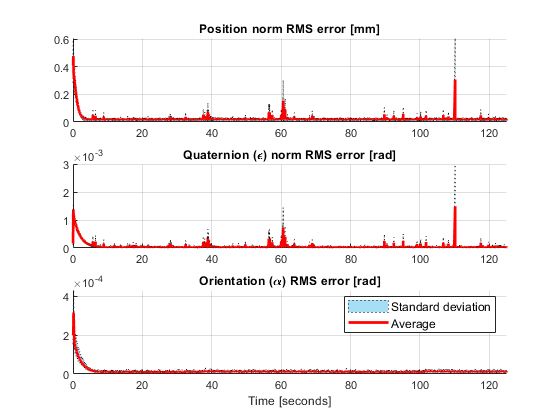}
    \caption{Position and orientation RMS errors and standard deviation for the bell-mouth section of the intake funnel.}
    \label{fig:intake-funnel-curved-mean-errors}
\end{figure}

The final produced part is shown in \reffig{intake-funnel}(a). The final part presented good symmetry, which can be attested by a 3D scan performed in the part and shown in \reffig{intake-funnel}(b).
\begin{figure}[ht!]
    \centering
    \subfloat[]{\includegraphics[width=0.38\columnwidth,trim={35cm 10cm 35cm 10cm},clip,angle=-90,origin=c]{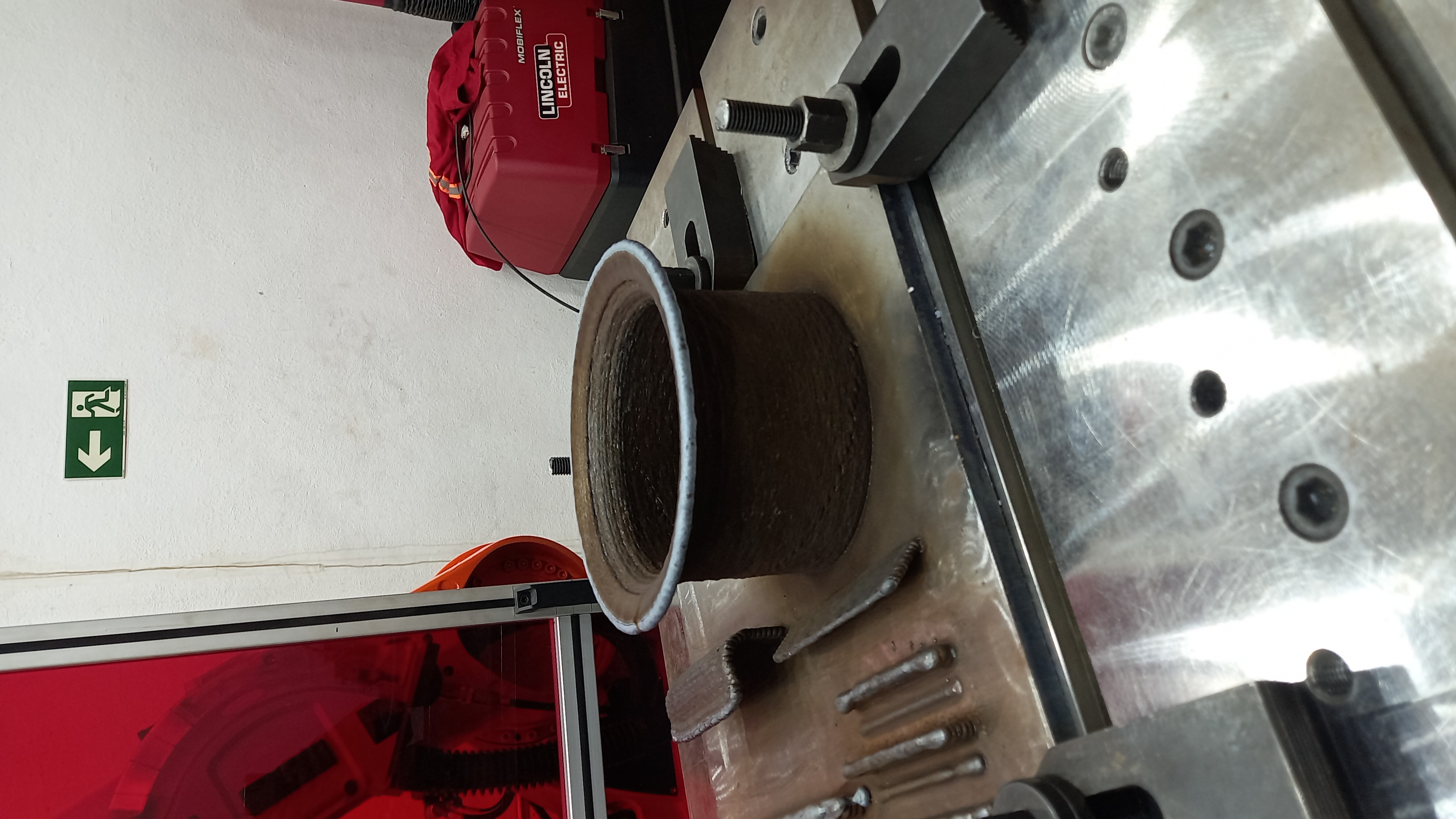}}
    \label{fig:intake-funnel-sub}
    \hspace{0.3cm}
    \raisebox{-0.11in}{\subfloat[]{\includegraphics[width=0.37\columnwidth,trim={9cm 2cm 8cm 2cm},clip]{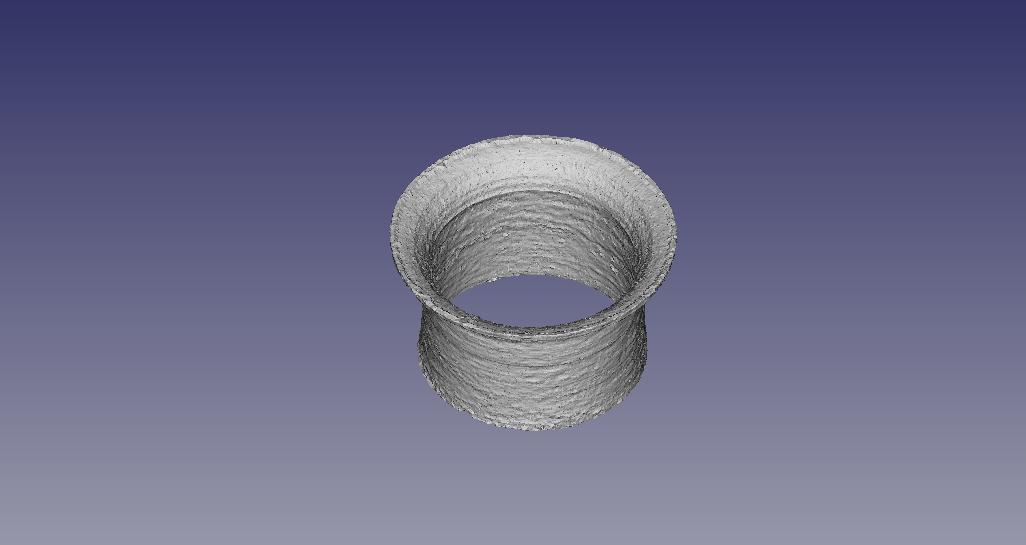}}}
    \label{fig:intake-funnel-scan}
    \caption{(a) Final manufactured intake funnel and (b) intake funnel profilometer scan.}
    \label{fig:intake-funnel}
\end{figure}
%%%%%%%%%%%%%%
%%%%%%%%%%%%%%
%%%%%%%%%%%%%%
\section{Concluding Remarks and Perspectives}
\label{sec:conclusions}
This works addresses the online coordinated control problem for WAAM applications regarding the manufacturing of non-constant build direction parts. The implemented control algorithm proved to be a reliable alternative to the task-priority method proposed in \cite{Lizarralde2022}, but with a more general approach based on the task augmentation method. The asymptotic stability  of the closed-loop system considering the augmented Jacobian approach is demonstrated considering unmodeled dynamics inherent of the robot’s internal controllers.

An algorithm singularity is verified when deposition frame and inertial frame $z-$axes are aligned, a very common configuration for WAAM systems. This causes a redundancy between the deposition task and the TCP $z$-axis alignment task which induces the augmented Jacobian to lose rank. This singularity is handled by using the DLS method with numerical filtering without affecting the deposition task.

Future works include expressing the coordination control for WAAM as an optimization problem, using slack variables to ease the trajectory tracking of some degrees of freedom allowing the obstacle avoidance for depositions in non-regular substrate or cladding of cylindrical parts. The authors also intend to investigate the use of the strategy for higher redundancy levels systems aimed at hybrid manufacturing applications.

\appendices
\appendix{Proof of theorem \ref{theo:extended_jacobian}}
\label{sec:appendix}
Consider the Lyapunov candidate function
\begin{align*}
V \!=\! \frac{1}{2}{\bm{e}_{p}}\trans \bm{e}_{p} \!+\! (e_{\eta}\!-\!1)^2 \!+\! \bm{e}_{\bm{\epsilon}}\trans \bm{e}_{\bm{\epsilon}} \!+\! (e_{\eta s}\!-\!1)^2 \!+\! {\bm{e}_{s}}\trans \bm{e}_{s} \!\\
+ \alpha (L_m\!-\! \int_0^t \bm{\eta}\trans {J_A}\trans J_A \bm{\eta} d\tau)
\end{align*}

The time derivative of $V$ is given by:
\begin{align}
\dot V = \bm{e}_{p}\trans \dot{\bm{e}}_{p} + 2(e_{\eta}-1)\dot{e}_{\eta} +2\bm{e}_{\bm{\epsilon}}\trans\dot{\bm{e}}_{\bm{\epsilon}} + 2(e_{\eta s}-1)\dot{\bm{e}}_{s} \notag\\+ 2\bm{e}_{s}\trans\dot{\bm{e}}_{s} - \alpha \bm{\eta}\trans {J_A}\trans J_A \bm{\eta}
\end{align}

From the quaternion propagation error \refeq{propagation_constrained}:
\begin{equation*}
\begin{array}{lll}
    \dot{e}_{\eta} = -\frac{1}{2}{\bm{e}_{\bm{\epsilon}}}\trans \tilde{\bm{\omega}},&&
    \dot{\bm{e}}_{\bm{\epsilon}} = \frac{1}{2}e_{\eta}\tilde{\bm{\omega}}-\bm{e}_{\bm{\epsilon}}\times\tilde{\bm{\omega}},\\
    \dot{e}_{\eta s} = -\frac{1}{2}{\bm{e}_{s}}\trans \tilde{\bm{\omega}}_s,&&%extra space
    \dot{\bm{e}}_s\ = \frac{1}{2}\left(e_{\eta s}\tilde{\bm{\omega}}_s+\begin{bmatrix}
        e_{\epsilon_2 s} \\-e_{\epsilon_1 s}
    \end{bmatrix}\tilde{\omega}^b_z\right)
\end{array}
\end{equation*}

Therefore:
\begin{align}
\dot V = \bm{e}_{p}\trans \dot{\bm{e}}_{p} &+ 
\begin{bmatrix} (e_{\eta}\!-\!1)\\ \bm{e}_{\bm{\epsilon}}  \end{bmatrix}\trans 
\begin{bmatrix}
        -\bm{e}_{\bm{\epsilon}}\trans \\ e_{\eta} I - \hat{\bm{e}}_{\bm{\epsilon}}
\end{bmatrix} \tilde{\bm{\omega}} 
\nonumber\\&+%
\begin{bmatrix} (e_{\eta s}\!-\!1)\\ \bm{e}_{s}  \end{bmatrix}\trans 
\begin{bmatrix}
        -\bm{e}_{s}\trans \tilde{\bm{\omega}}_s \\ e_{\eta s}\tilde{\bm{\omega}}_s+\begin{bmatrix}
        e_{\epsilon_2 s} \\-e_{\epsilon_1 s}
    \end{bmatrix}\tilde{\omega}^t_z
\end{bmatrix} 
\nonumber\\&
- \alpha \bm{\eta}\trans {J_A}\trans J_A \bm{\eta}
\end{align}

Since $e_{\eta}\bm{e}_{\bm{\epsilon}}\trans \!=\! \bm{e}_{\bm{\epsilon}}\trans e_{\eta}$, 
$\bm{e}_{\bm{\epsilon}}\trans \hat{\bm{e}}_{\bm{\epsilon}} \!=\! 0$, 
$e_{\eta s}\bm{e}_{s}\trans \!=\! \bm{e}_{s}\trans e_{\eta s}$, 
$\bm{e}_{s}\trans \hat{\bm{e}}_{s}=0$
and $\bm{e}_s\trans[e_{\epsilon_2 s}\quad -e_{\epsilon_1 s}]\trans=0$ 
, then:
\begin{equation}
\dot V = \begin{bmatrix} \bm{e}_{p} \\ \bm{e}_{\bm{\epsilon}} \\\bm{e}_{s}  \end{bmatrix}\trans \begin{bmatrix} \dot{\bm{e}}_{p} \\ \tilde{\bm{\omega}} \\ \tilde{\bm{\omega}}_s \end{bmatrix}  - \alpha \bm{\eta}\trans J_A\trans J_A \bm{\eta}
\label{eq:lyapunov-dot2}
\end{equation}

Consequently, substituting \refeq{errordynamics} into \refeq{lyapunov-dot2}:
\begin{align}
\dot V = -\!\begin{bmatrix} \bm{e}_{p} \\ \bm{e}_{\bm{\epsilon}} \\ \bm{e}_{s}  \end{bmatrix}\trans \!
\begin{bmatrix} K_p & 0 & 0 \\ 0 & K_o & 0\\ 0 & 0 & K_s\end{bmatrix} \!
\begin{bmatrix} \bm{e}_{p} \\ \bm{e}_{\bm{\epsilon}} \\\bm{e}_{s}  \end{bmatrix}  \!\notag\\-\!
\begin{bmatrix} \bm{e}_{p} \\ \bm{e}_{\bm{\epsilon}}  \\\bm{e}_{s}\end{bmatrix}\trans\! J_A \bm{\eta}  \!-\! \alpha \bm{\eta}\trans J_A\trans J_A \bm{\eta}
\end{align}

Thus, using Young's inequality \cite{khalil02}, it is always possible to choose positive constants $k_1, k_2, k_3, k_4$ such that the following inequality is held:
\begin{equation}
\dot V \leq -k_1 \bm{e}_{p}\trans \bm{e}_{p} - k_2 \bm{e}_{\bm{\epsilon}}\trans \bm{e}_{\bm{\epsilon}} - k_3 \bm{e}_s\trans\bm{e}_s - k_4  \bm{\eta}\trans J\trans J \bm{\eta}  \leq 0   
\end{equation}
Therefore, $\dot{V} \!\leq\! 0$ 
and consequently by Lyapunov Theory \cite{khalil02},
$\bm{e}_{p}, e_q, \bm{e}_{s}, J\bm{\eta} \!\in\! \cal{L}_\infty$. 
Besides that, given that $\bm{\eta} \!\in\! \mathcal{L}_{2} \!\cap\! \cal{L}_\infty$,
then $J\bm{\eta} \!\in\! \mathcal{L}_{2} \cap \cal{L}_\infty$.
Given that $\bm{e}_{p}, e_q, \bm{e}_{s}, J\bm{\eta}$ are bounded, $\ddot V$ is also bounded, then $\dot V$ is uniformly continuous. 
By Barbalat's Lemma \cite{khalil02} it is possible to conclude that $\dot V \!\rightarrow\! 0$, and consequently $\bm{e}_{p}\! \rightarrow \!0$, $\bm{e}_{\bm{\epsilon}}\!\rightarrow\! 0,  e_{\eta} \!\rightarrow\! \pm 1$, $\bm{e}_{s}\!\rightarrow\! 0$ and $e_{\eta s} \!\rightarrow\! \pm 1$  for $t \!\rightarrow\! \infty$.

 \bibliographystyle{elsarticle-num} 
 %\bibliography{tase-bib}

\end{document}